\documentclass{article} \usepackage{matus}

\def\baru{\bar{u}}
\def\cR{\mathcal{R}}
\def\nR{\nabla\cR}
\def\wnn{w_{\mathrm{prod}}}
\def\llog{\ell_{\log}}
\def\lexp{\ell_{\exp}}
\def\tr{\mathrm{tr}}

\title{Gradient descent aligns the layers of deep linear networks}

\author{Ziwei Ji\qquad Matus Telgarsky\\
\tt{\{ziweiji2,mjt\}@illinois.edu}\\
University of Illinois, Urbana-Champaign}
\date{}

\begin{document}

\maketitle

\begin{abstract}
This paper establishes risk convergence and
asymptotic weight matrix alignment
---
  a form of implicit regularization
---
of gradient flow and gradient descent when applied to deep linear networks
on linearly separable data.
In more detail, for gradient flow applied to strictly decreasing
loss functions (with similar results for gradient descent with
particular decreasing step sizes):
(i) the risk converges to $0$;
(ii) the normalized $i^\text{th}$ weight matrix asymptotically equals its
rank-$1$ approximation $u_iv_i^\top$;
(iii) these rank-$1$ matrices are
  aligned across layers, meaning $|v_{i+1}^\top u_i|\to1$.
  In the case of the logistic loss (binary cross entropy), more
  can be said: the linear function induced by the network ---
  the product of its weight matrices ---
  converges to the same direction as the maximum margin solution.
  This last property was identified in prior work,
  but only under assumptions on gradient descent which
  here are implied by the alignment phenomenon.
\end{abstract}

\section{Introduction}

Efforts to explain the effectiveness of gradient descent in deep learning
have uncovered an exciting possibility: it not only finds solutions with low error,
but also biases the search for low complexity solutions which generalize well
\citep{rethinking,spectrally_normalized,nati_iclr,nati_nips}.

This paper analyzes the implicit regularization of gradient descent
and gradient flow on deep linear networks and linearly separable data.
For strictly decreasing losses, the optimum is at infinity,
and
we establish various \emph{alignment phenomena}:
\begin{itemize}
  \item
    For each weight matrix $W_i$, the corresponding normalized weight matrix
    $\nicefrac{W_i}{\|W_i\|_F}$
    asymptotically equals its rank-$1$ approximation $u_iv_i^\top$,
    where the Frobenius norm $\|W_i\|_F$ satisfies $\|W_i\|_F\to\infty$.
    In other words,
    $\nicefrac{\|W_i\|_2 }{ \|W_i\|_F } \to 1$,
    and asymptotically only the rank-$1$ approximation of $W_i$
    contributes to the final predictor,
    a form of implicit regularization.

  \item
    Adjacent rank-$1$ weight matrix approximations
    are aligned:
    $|v_{i+1}^\top{}u_i| \to 1$.

  \item
    For the logistic loss, the first right singular vector $v_1$ of $W_1$ is aligned with the data, meaning $v_1$ converges
    to the unique maximum margin predictor $\bar u$ defined by the data.
    Moreover, the linear predictor induced by the network, $\wnn := W_L\cdots W_1$, is also aligned with the
    data, meaning $\nicefrac{\wnn }{ \|\wnn\| }\to \bar u$.
\end{itemize}

Simultaneously, this work proves that the risk is globally optimized: it asymptotes to 0.
Alignment and risk convergence are proved simultaneously;
the phenomena are coupled within the proofs.

Since the layers align, they can be viewed as a \emph{minimum norm solution}:
they do not ``waste norm'' on components which are killed off when the layers are
multiplied together.
Said another way, given data $((x_i,y_i))_{i=1}^n$,
the normalized matrices $(\nicefrac {W_1}{\|W_1\|_F}, \ldots, \nicefrac{W_L}{\|W_L\|_F})$
asymptotically solve a maximum margin problem which demands \emph{all} weight matrices be small, not merely their product:
  \[
    \max_{\substack{W_L \in \R^{1\times d_{L-1}}\\ \|W_L\|_F = 1}}
    \cdots
    \max_{\substack{W_1 \in \R^{d_1\times d_0}\\ \|W_1\|_F = 1}}
    \quad
    \min_i
    y_i(W_L \cdots W_1)x_i.
  \]

\begin{figure}[h!]
  \centering
  \begin{subfigure}[b]{0.495\textwidth}
    \includegraphics[width = 1.0\textwidth]{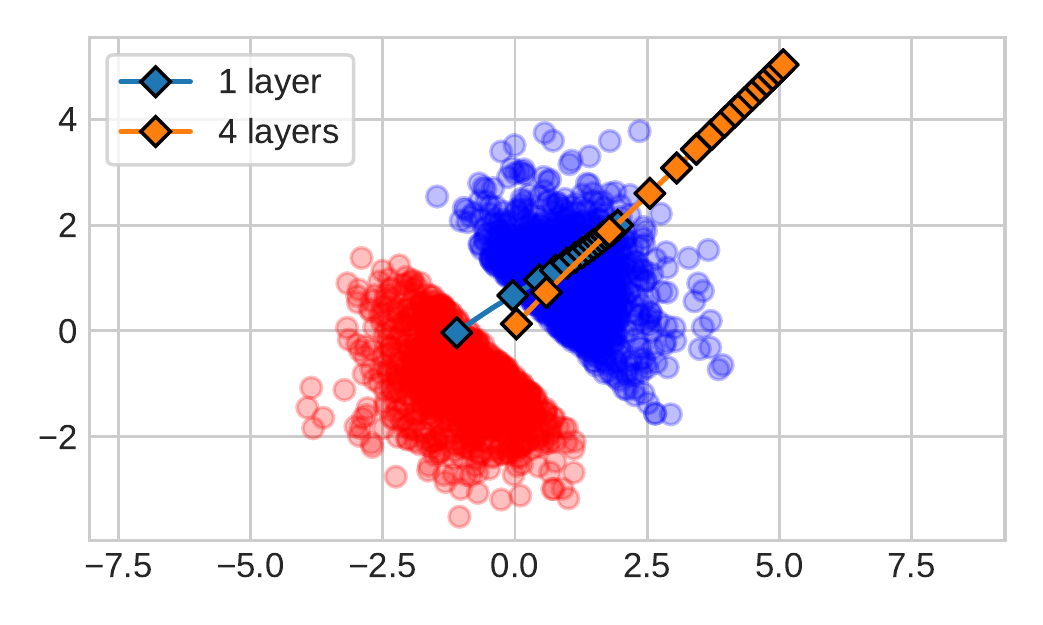}
    \caption{Margin maximization.\label{fig:maxmarg}}
  \end{subfigure}
  \begin{subfigure}[b]{0.495\textwidth}
    \includegraphics[width = 1.0\textwidth]{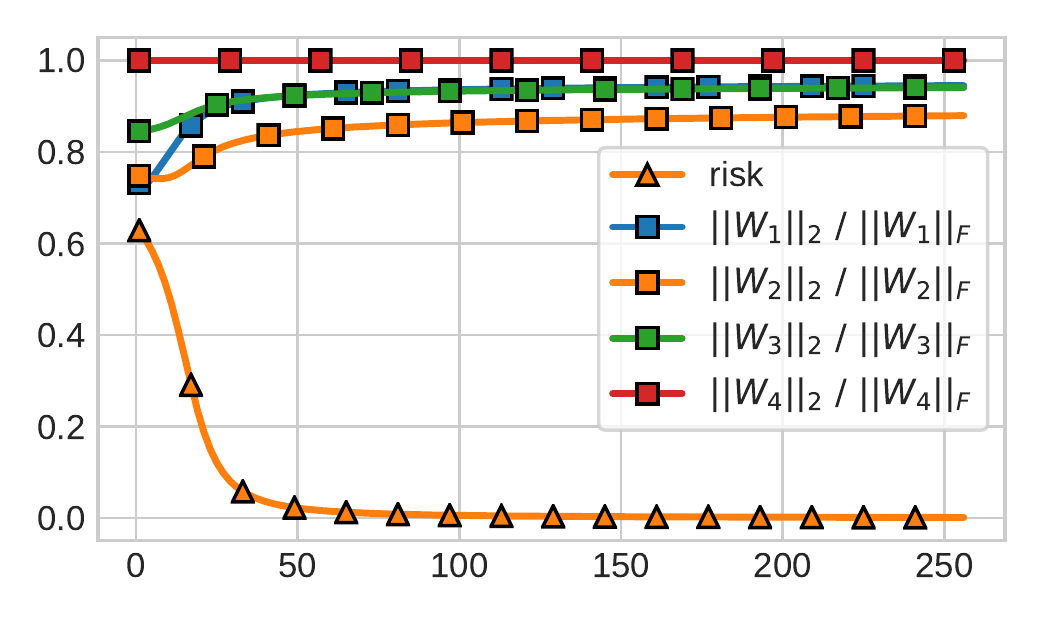}
    \caption{Alignment and risk minimization.\label{fig:align}}
  \end{subfigure}
  \caption{Visualization of margin maximization and self-regularization of layers on synthetic data with a $4$-layer linear network compared to a $1$-layer network (a linear predictor).
  \Cref{fig:maxmarg} shows the convergence of $1$-layer and $4$-layer networks to the same
  margin-maximizing linear predictor on positive (blue) and negative (red) separable data.
  \Cref{fig:align} shows the convergence of $\|W_i\|_2/\|W_i\|_F$ to 1 on each layer,
  plotted against the risk.
}
  \label{fig:intro:margin_dists}
\end{figure}

The paper is organized as follows.
This introduction continues with related work, notation, and assumptions
in \Cref{sec:related,sec:notation}.
The analysis of gradient flow is in \Cref{sec:gf},
and gradient descent is analyzed in \Cref{sec:gd}.
The paper closes with future directions in \Cref{sec:future};
a particular highlight is a preliminary experiment
on CIFAR-10 which establishes empirically that a form of the alignment phenomenon occurs on
the standard nonlinear network AlexNet.

\begin{figure}[h!]
  \centering
  \begin{subfigure}[b]{0.495\textwidth}
    \includegraphics[width = 1.0\textwidth]{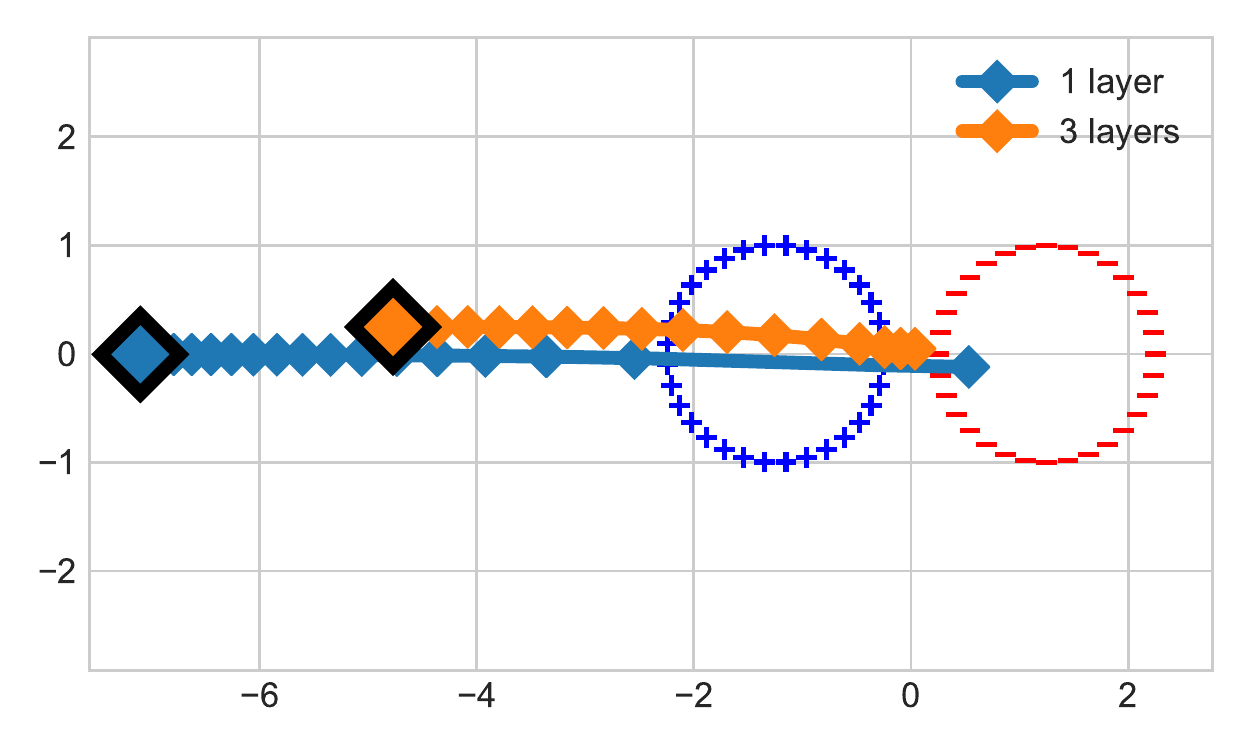}
    \caption{Overall margin maximization.\label{fig:il}}
  \end{subfigure}
  \begin{subfigure}[b]{0.495\textwidth}
    \includegraphics[width = 1.0\textwidth]{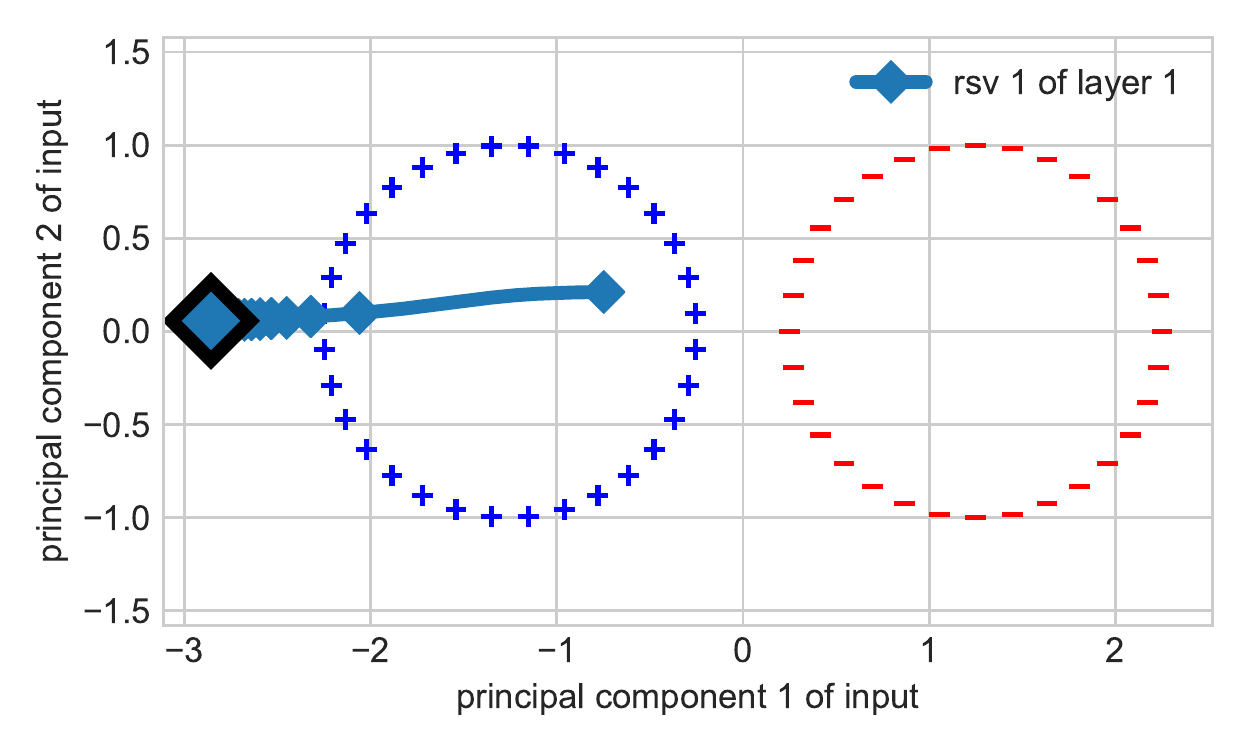}
    \caption{Margin maximization at layer 1.\label{fig:il:0}}
  \end{subfigure}
  \begin{subfigure}[b]{0.495\textwidth}
    \includegraphics[width = 1.0\textwidth]{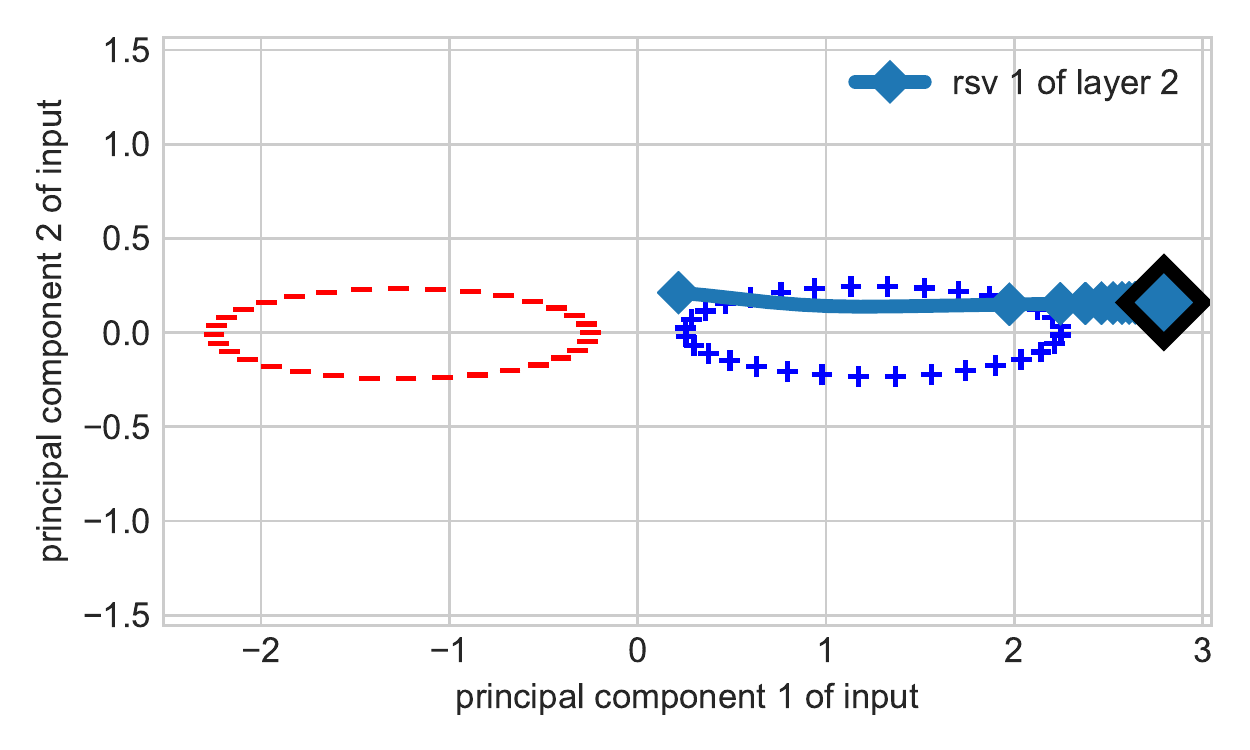}
    \caption{Margin maximization at layer 2.\label{fig:il:1}}
  \end{subfigure}
  \begin{subfigure}[b]{0.495\textwidth}
    \includegraphics[width = 1.0\textwidth]{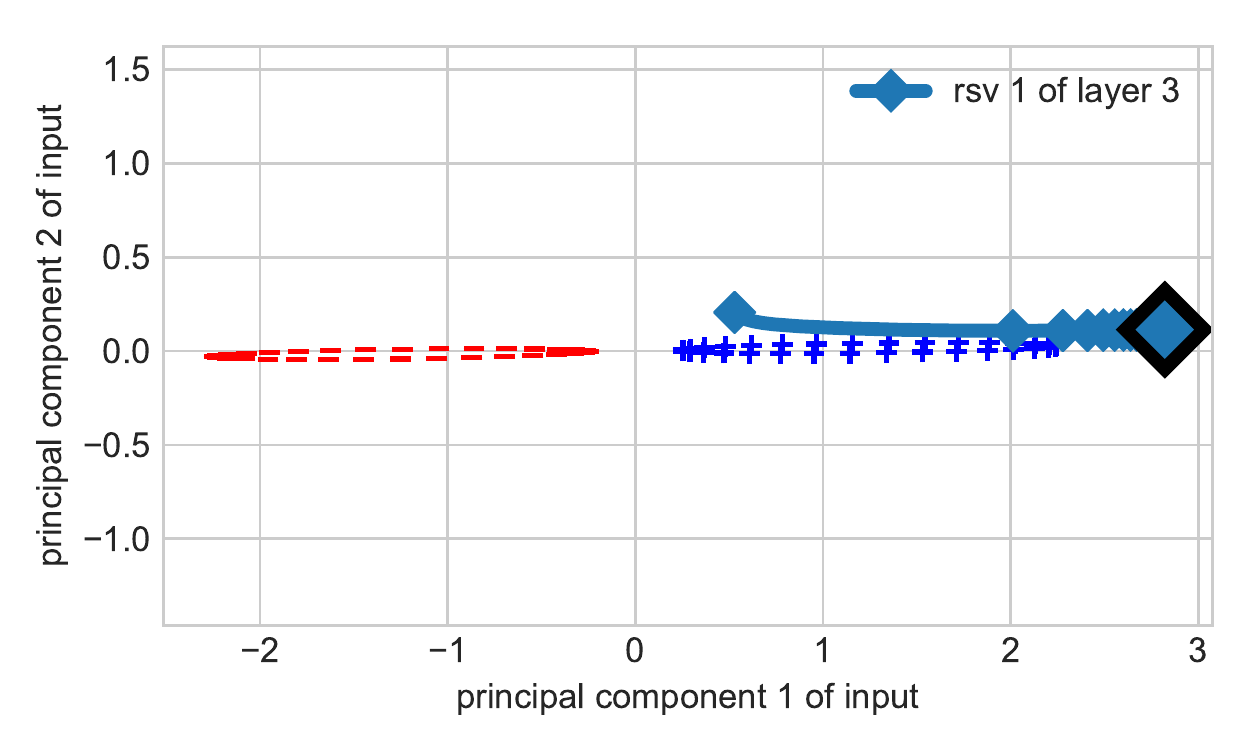}
    \caption{Margin maximization at layer 3.\label{fig:il:2}}
  \end{subfigure}
  \caption{A visualization of inter-layer alignment on data consisting
    of two well-separated circles with a 3-layer linear network.
    \Cref{fig:il} depicts, as in \Cref{fig:maxmarg},
    that optimizing 1- and 3-layer linear networks finds the same maximum margin solution.
    The other three plots show the data as it is mapped through progressively more and more layers.
    Due to alignment, the product $W_i\cdots{}W_1$ becomes $u_i\baru^\top$, where $u_i$ is the top left
    singular vector of $W_i$, which means that asymptotically the mapped data will be well
    separated and lie along the span of $u_i$, as depicted by the flattening in \Cref{fig:il:0,fig:il:1,fig:il:2}.
    Additionally, these three subfigures show that the top right singular vector $v_{i+1}$ of the
    subsequent layer is aligned with this $u_i$, which in these plots (with principal component axes)
    corresponds to following a horizontal line.
  }
  \label{fig:interlayer_align}
\end{figure}

\subsection{Related work}
\label{sec:related}

On the implicit regularization of gradient descent, \citet{nati_iclr} show that for linear predictors and linearly separable data, the gradient descent iterates converge to the same direction as the maximum margin solution. \citet{ours} further characterize such an implicit bias for general nonseparable data. \citet{nati_nips} consider gradient descent on fully connected linear networks and linear convolutional networks. In particular, for the exponential loss, assuming the risk is minimized to $0$ and the gradients converge in direction, they show that the whole network converges in direction to the maximum margin solution. These two assumptions are on the gradient descent process itself, and specifically the second one might be hard to interpret and justify. Compared with \citet{nati_nips}, this paper \emph{proves} that the risk converges to $0$ and the weight matrices align; moreover the proof here proves the properties simultaneously, rather than assuming one and deriving the other. Lastly, \citet{arora_icml} show for deep linear networks (and later \citet{jason_nips} for ReLU networks) that gradient flow does not change the difference between squared Frobenius norms of any two layers.  We use a few of these tools in our proofs; please see \Cref{sec:gf,sec:gd} for details.

For a smooth (nonconvex) function, \citet{jason_gd} show that any strict saddle can be avoided almost surely with small step sizes. If there are only countably many saddle points and they are all strict, and if gradient descent iterates converge, then this implies (almost surely) they converge to a local minimum. In the present work, since there is no finite local minimum, gradient descent will go to infinity and never converge, and thus these results of \citet{jason_gd} do not show that the risk converges to $0$.

There has been a rich literature on linear networks. \citet{saxe} analyze the learning dynamics of deep linear networks, showing that they exhibit some learning patterns similar to nonlinear networks, such as a long plateau followed by a rapid risk drop. \citet{arora_icml} show that depth can help accelerate optimization. On the landscape properties of deep linear networks, \citet{kawaguchi,laurent_icml} show that under various structural assumptions, all local optima are global. \citet{zhou_iclr} give a necessary and sufficient characterization of critical points for deep linear networks.

\subsection{Notation, setting, and assumptions}
\label{sec:notation}

Consider a data set $\{(x_i,y_i)\}_{i=1}^n$, where $x_i\in \mathbb{R}^d$, $\|x_i\|\le1$, and $y_i\in\{-1,+1\}$. The data set is assumed to be linearly separable, i.e., there exists a unit vector $u$ which correctly classifies every data point: for any $1\le i\le n$, $y_i\langle u,x_i\rangle>0$. Furthermore, let $\gamma:=\max_{\|u\|=1}\min_{1\le i\le n}y_i\langle u,x_i\rangle>0$ denote the maximum margin, and $\baru := \argmax_{\|u\|=1}\min_{1\le i\le n}y_i\langle u,x_i\rangle$ denote the maximum margin solution (the solution to the hard-margin SVM).

A linear network of depth $L$ is parameterized by weight matrices $W_L,\ldots,W_1$, where $W_k\in \mathbb{R}^{d_k\times d_{k-1}}$, $d_0=d$, and $d_L=1$. Let $W=(W_L,\ldots,W_1)$ denote all parameters of the network. The (empirical) risk induced by the network is given by
\begin{equation*}
    \cR(W)=\cR\del{W_L,\ldots,W_1}=\frac{1}{n}\sum_{i=1}^{n}\ell\del{y_iW_L\cdots W_1x_i}=\frac{1}{n}\sum_{i=1}^{n}\ell\del{\langle\wnn,z_i\rangle},
\end{equation*}
where $\wnn := (W_L\cdots W_1)^{\top}$, and $z_i:=y_ix_i$.

The loss $\ell$ is assumed to be continuously differentiable, unbounded, and strictly decreasing to $0$. Examples include the exponential loss $\lexp(x)=e^{-x}$ and the logistic loss $\llog(x)=\ln\del{1+e^{-x}}$.
\begin{assumption}\label{ass:loss}
    $\ell'<0$ is continuous, $\lim_{x\to-\infty}\ell(x)=\infty$ and $\lim_{x\to\infty}\ell(x)=0$.
\end{assumption}

This paper considers gradient flow and gradient descent,
where gradient flow $\cbr{W(t)\middle|t\ge0,t\in \mathbb{R}}$ can be interpreted as gradient descent with infinitesimal step sizes. It starts from some $W(0)$ at $t=0$, and proceeds as
\begin{equation*}
    \frac{\dif W(t)}{\dif t}=-\nR\del{W(t)}.
\end{equation*}
By contrast, gradient descent $\cbr{W(t)\middle|t\ge0,t\in \mathbb{Z}}$ is a discrete-time process given by
\begin{equation*}
    W(t+1)=W(t)-\eta_t\nR\del{W(t)},
\end{equation*}
where $\eta_t$ is the step size at time $t$.

We assume that the initialization of the network is not a critical point and induces a risk no larger than the risk of the trivial linear predictor $0$.
\begin{assumption}\label{ass:init}
    The initialization $W(0)$ satisfies $\nR\del{W(0)}\ne0$ and $\cR\del{W(0)}\le\cR(0)=\ell(0)$.
\end{assumption}
It is natural to require that the initialization is not a critical point, since otherwise gradient flow/descent will never make a progress. The requirement $\cR\del{W(0)}\le\cR(0)$ can be easily satisfied, for example, by making $W_1(0)=0$ and $W_L(0)\cdots W_2(0)\ne0$. On the other hand, if $\cR\del{W(0)}>\cR(0)$, gradient flow/descent may never minimize the risk to $0$. Proofs of those claims are given in \Cref{sec:notation_app}.

\section{Results for gradient flow}\label{sec:gf}

In this section, we consider gradient flow. Although impractical when compared with gradient descent, gradient flow can simplify the analysis and highlight proof ideas. For convenience, we usually use $W$, $W_k$, and $\wnn$, but they all change with (the continuous time) $t$. Only proof sketches are given here; detailed proofs are deferred to \Cref{sec:gf_app}.

\subsection{Risk convergence}

One key property of gradient flow is that it never increases the risk:
\begin{equation}\label{eq:gf_risk_no_inc}
    \frac{\dif\cR(W)}{\dif t}=\left\langle\nR(W),\frac{\dif W}{\dif t}\right\rangle=-\|\nR(W)\|^2=-\sum_{k=1}^{L}\enVert{\frac{\partial\cR}{\partial W_k}}_F^2\le0.
\end{equation}
We now state the main result: under \Cref{ass:loss,ass:init}, gradient flow minimizes the risk, $W_k$ and $\wnn$ all go to infinity,
and the alignment phenomenon occurs.
\begin{theorem}\label{thm:gf_risk}
    Under \Cref{ass:loss,ass:init}, gradient flow iterates satisfy the following properties:
    \begin{itemize}
        \item $\lim_{t\to\infty}\cR(W)=0$.
        \item For any $1\le k\le L$, $\lim_{t\to\infty}\|W_k\|_F=\infty$.
        \item For any $1\le k\le L$, letting $(u_k,v_k)$ denote the first left and right singular vectors of $W_k$,
        \begin{equation*}
            \lim_{t\to\infty}\enVert{\frac{W_k}{\|W_k\|_F}-u_kv_k^{\top}}_F=0.
        \end{equation*}
        Moreover, for any $1\le k<L$, $\lim_{t\to\infty}\envert{\langle v_{k+1},u_k\rangle}=1$. As a result,
        \begin{equation*}
            \lim_{t\to\infty}\envert{\left\langle \frac{\wnn}{\prod_{k=1}^L\|W_k\|_F},v_1\right\rangle}=1,
        \end{equation*}
        and thus $\lim_{t\to\infty}\|\wnn\|=\infty$.
    \end{itemize}
\end{theorem}

\Cref{thm:gf_risk} is proved using two lemmas, which may be of independent interest. To show the ideas, let us first introduce a little more notation. Recall that $\cR(W)$ denotes the empirical risk induced by the deep linear network $W$. Abusing the notation a little, for any linear predictor $w\in \mathbb{R}^d$, we also use $\cR(w)$ to denote the risk induced by $w$. With this notation, $\cR(W)=\cR(\wnn)$, while
\begin{equation*}
    \nR(\wnn)=\frac{1}{n}\sum_{i=1}^{n}\ell'\del{\langle\wnn,z_i\rangle}z_i=\frac{1}{n}\sum_{i=1}^{n}\ell'\del{W_L\cdots W_1z_i}z_i
\end{equation*}
is in $\mathbb{R}^{d}$ and different from $\nR(W)$, which has $\sum_{k=1}^{L}d_{k}d_{k-1}$ entries, as given below:
\begin{equation*}
    \frac{\partial\cR}{\partial W_k}=W_{k+1}^{\top}\cdots W_L^{\top}\nR(\wnn)^{\top}W_1^{\top}\cdots W_{k-1}^{\top}.
\end{equation*}
Furthermore, for any $R>0$, let
\begin{equation*}
    B(R)=\cbr{W\middle|\max_{1\le k\le L}\|W_k\|_F\le R}.
\end{equation*}

The first lemma shows that for any $R>0$, the time spent by gradient flow in $B(R)$ is finite.
\begin{lemma}\label{lem:gf_unbounded}
    Under \Cref{ass:loss} and \ref{ass:init}, for any $R>0$, there exists a constant $\epsilon(R)>0$, such that for any $t\ge1$ and any $W\in B(R)$, $\|\partial\cR/\partial W_1\|_F\ge\epsilon(R)$. As a result, gradient flow spends a finite amount of time in $B(R)$ for any $R>0$, and $\max_{1\le k\le L}\|W_k\|_F$ is unbounded.
\end{lemma}
Here is the proof sketch. If $\|W_k\|_F$ are bounded, then $\|\nR(\wnn)\|$ will be lower bounded by a positive constant, therefore if $\|\partial\cR/\partial W_1\|_F=\|W_L\cdots W_2\|\|\nR(\wnn)\|$ can be arbitrarily small, then $\|W_L\cdots W_2\|$ and $\|\wnn\|$ can also be arbitrarily small, and thus $\cR(W)$ can be arbitrarily close to $\cR(0)$. This cannot happen after $t=1$, otherwise it will contradict \Cref{ass:init} and \cref{eq:gf_risk_no_inc}.

To proceed, we need the following properties of linear networks from prior work \citep{arora_icml,jason_nips}. For any time $t\ge0$ and any $1\le k<L$,
\begin{equation}\label{eq:gf_lnn_spec}
    W_{k+1}^{\top}(t)W_{k+1}(t)-W_{k+1}^{\top}(0)W_{k+1}(0)=W_k(t)W_k^{\top}(t)-W_k(0)W_k^{\top}(0).
\end{equation}
To see this, just notice that
\begin{equation*}
    W_{k+1}^{\top}\frac{\partial\cR}{\partial W_{k+1}}=W_{k+1}^{\top}\cdots W_L^{\top}\nR(\wnn)^{\top}W_1^{\top}\cdots W_k^{\top}=\frac{\partial\cR}{\partial W_k}W_k^{\top}.
\end{equation*}
Taking the trace on both sides of \cref{eq:gf_lnn_spec}, we have
\begin{equation}\label{eq:gf_lnn_norm}
    \enVert{W_{k+1}(t)}_F^2-\enVert{W_{k+1}(0)}_F^2=\enVert{W_k(t)}_F^2-\enVert{W_k(0)}_F^2.
\end{equation}
In other words, the difference between the squares of Frobenius norms of any two layers remains a constant. Together with \Cref{lem:gf_unbounded}, it implies that all $\|W_k\|_F$ are unbounded.

However, even if $\|W_k\|_F$ are large, it does not follow necessarily that $\|\wnn\|$ is also large. \Cref{lem:gf_align} shows that this is indeed true: for gradient flow, as $\|W_k\|_F$ get larger, adjacent layers also get more aligned to each other, which ensures that their product also has a large norm.

For $1\le k\le L$, let $\sigma_k$, $u_k$, and $v_k$ denote the first singular value (the $2$-norm), the first left singular vector, and the first right singular vector of $W_k$, respectively. Furthermore,
define
\begin{align*}
    D &:=\del{\max_{1\le k\le L}\|W_k(0)\|_F^2}-\|W_L(0)\|_F^2+\sum_{k=1}^{L-1}\enVert{W_k(0)W_k^{\top}(0)-W_{k+1}^{\top}(0)W_{k+1}(0)}_2,
\end{align*}
which depends only on the initialization. If for any $1\le k<L$, $W_k(0)W_k^{\top}(0)=W_{k+1}^{\top}(0)W_{k+1}(0)$, then $D=0$.
\begin{lemma}\label{lem:gf_align}
    The gradient flow iterates satisfy the following properties:
    \begin{itemize}
        \item For any $1\le k\le L$, $\|W_k\|_F^2-\|W_k\|_2^2\le D$.
        \item For any $1\le k<L$, $\langle v_{k+1},u_k\rangle^2\ge1-\nicefrac{\del[1]{D+\|W_{k+1}(0)\|_2^2+\|W_k(0)\|_2^2}}{\|W_{k+1}\|_2^2}$.
        \item Suppose $\max_{1\le k\le L}\|W_k\|_F\to\infty$, then $\envert{\left\langle \nicefrac{\wnn}{\prod_{k=1}^{L}\|W_k\|_F}, v_1\right\rangle}\to1$.
    \end{itemize}
\end{lemma}
The proof is based on \cref{eq:gf_lnn_spec} and \cref{eq:gf_lnn_norm}. If $W_k(0)W_k^{\top}(0)=W_{k+1}^{\top}(0)W_{k+1}(0)$, then \cref{eq:gf_lnn_spec} gives that $W_{k+1}$ and $W_k$ have the same singular values, and $W_{k+1}$'s right singular vectors and $W_k$'s left singular vectors are the same. If it is true for any two adjacent layers, since $W_L$ is a row vector, all layers have rank $1$. With general initialization, we have similar results when $\|W_k\|_F$ is large enough so that the initialization is negligible. Careful calculations give the exact results in \Cref{lem:gf_align}.

An interesting point is that the implicit regularization result in \Cref{lem:gf_align} helps establish risk convergence in \Cref{thm:gf_risk}. Specifically, by \Cref{lem:gf_align}, if all layers have large norms, $\|W_L\cdots W_2\|$ will be large. If the risk is not minimized to $0$, $\|\nR(\wnn)\|$ will be lower bounded by a positive constant, and thus $\|\partial\cR/\partial W_1\|_F=\|W_L\cdots W_2\|\|\nR(\wnn)\|$ will be large. Invoking \cref{eq:gf_risk_no_inc}, \Cref{lem:gf_unbounded} and \cref{eq:gf_lnn_norm} gives a contradiction. Since the risk has no finite optimum, $\|W_k\|_F\to\infty$.

\subsection{Convergence to the maximum margin solution}

Here we focus on the exponential loss $\lexp(x)=e^{-x}$ and the logistic loss $\llog(x)=\ln(1+e^{-x})$. In addition to risk convergence, these two losses also enable gradient descent to find the maximum margin solution.

To get such a strong convergence, we need one more assumption on the data set. Recall that $\gamma=\max_{\|u\|=1}\min_{1\le i\le n}\langle u,z_i\rangle>0$ denotes the maximum margin, and $\baru$ denotes the unique maximum margin predictor which attains this margin $\gamma$. Those data points $z_i$ for which $\langle\baru,z_i\rangle=\gamma$ are called support vectors.
\begin{assumption}\label{ass:data}
    The support vectors span the whole space $\mathbb{R}^d$.
\end{assumption}
\Cref{ass:data} appears in prior work \citet{nati_iclr}, and can be satisfied in many cases: for example, it is almost surely true if the number of support vectors is larger than or equal to $d$ and the data set is sampled from some density w.r.t. the Lebesgue measure. It can also be relaxed to the situation that the support vectors and the whole data set span the same space; in this case $\nR(\wnn)$ will never leave this space, and we can always restrict our attention to this space.

With \Cref{ass:data}, we can state the main theorem.
\begin{theorem}\label{thm:gf_min_norm}
    Under \Cref{ass:init,ass:data}, for almost all data and for losses $\lexp$ and $\llog$,
    then $\lim_{t\to\infty}\envert{\langle v_1,\baru\rangle}=1$, where $v_1$ is the first right singular vector of $W_1$. As a result, $\lim_{t\to\infty}\nicefrac{\wnn}{\prod_{k=1}^L\|W_k\|_F}=\baru$.
\end{theorem}

Before summarizing the proof, we can simplify both theorems into the following \emph{minimum norm}
property mentioned in the introduction.
\begin{corollary}
  \label{fact:gf_min_norm_2}
    Under \Cref{ass:init,ass:data}, for almost all data and for losses $\lexp$ and $\llog$,
    \[
      \min_i y_i \del{ \frac {W_L}{\|W_L\|_F} \cdots \frac {W_1}{\|W_1\|_F} } x_i
      \quad\xrightarrow[t\to\infty]{}\quad
      \max_{\substack{A_L \in \R^{1 \times d_{L-1}}\\\|A_L\|_F = 1}}
      \cdots
      \max_{\substack{A_1 \in \R^{d_1 \times d_0}\\\|A_1\|_F = 1}}
      \min_i y_i \del{ A_L \cdots A_1 } x_i.
    \]
\end{corollary}

\Cref{thm:gf_min_norm} relies on two structural lemmas. The first one is based on a similar almost-all argument due to \citet[Lemma 12]{nati_iclr}. Let $S\subset\{1,\ldots,n\}$ denote the set of indices of support vectors.
\begin{lemma}\label{lem:strong_conv}
    Under \Cref{ass:data}, if the data set is sampled from some density w.r.t. the Lebesgue measure, then with probability $1$,
    \begin{equation*}
        \alpha:=\min_{|\xi|=1,\xi\perp\baru}\max_{i\in S}\langle\xi,z_i\rangle>0.
    \end{equation*}
\end{lemma}

Let $\baru^{\perp}$ denote the orthogonal complement of $\mathrm{span}(\baru)$, and let $\Pi_{\perp}$ denote the projection onto $\baru^{\perp}$. We prove that if $\|\Pi_{\perp}w\|$ is large enough, gradient flow starting from $w$ will tend to decrease $\|\Pi_{\perp}w\|$.
\begin{lemma}\label{lem:perp_bound}
    Under \Cref{ass:data}, for almost all data, $\lexp$ and $\llog$, and any $w\in \mathbb{R}^d$, if $\langle w,\bar{u}\rangle\ge0$ and $\|\Pi_{\perp}w\|$ is larger than $\nicefrac{1+\ln(n)}{\alpha}$ for $\lexp$ or $\nicefrac{2n}{e\alpha}$ for $\llog$, then $\langle \Pi_{\perp}w,\nR(w)\rangle\ge0$.
\end{lemma}

With \Cref{lem:strong_conv} and \Cref{lem:perp_bound} in hand, we can prove \Cref{thm:gf_min_norm}. Let $\Pi_{\perp}W_1$ denote the projection of rows of $W_1$ onto $\baru^{\perp}$. Notice that
\begin{equation*}
    \Pi_{\perp}\wnn=\del{W_L\ldots W_2(\Pi_{\perp}W_1)}^{\top}\quad\textrm{and}\quad \frac{\dif\|\Pi_{\perp}W_1\|_F^2}{\dif t}=-2\langle\Pi_{\perp}\wnn,\nR(\wnn)\rangle.
\end{equation*}
If $\|\Pi_{\perp}W_1\|_F$ is large compared with $\|W_1\|_F$, since layers become aligned, $\|\Pi_{\perp}\wnn\|$ will also be large, and then \Cref{lem:perp_bound} implies that $\|\Pi_{\perp}W_1\|_F$ will not increase. At the same time, $\|W_1\|_F\to\infty$, and thus for large enough $t$, $\|\Pi_{\perp}W_1\|_F$ must be very small compared with $\|W_1\|_F$. Many details need to be handled to make this intuition exact; the proof is given in \Cref{sec:gf_app}.

\section{Results for gradient descent}\label{sec:gd}

One key property of gradient flow which is used in the previous proofs is that it never increases the risk, which is not necessarily true for gradient descent. However, for smooth losses (i.e, with Lipschitz continuous derivatives), we can design some decaying step sizes, with which gradient descent never increases the risk, and basically the same results hold as in the gradient flow case. Deferred proofs are given in \Cref{sec:gd_app}.

We make the following additional assumption on the loss, which is satisfied by the logistic loss $\llog$.
\begin{assumption}\label{ass:smooth}
    $\ell'$ is $\beta$-Lipschitz (i.e, $\ell$ is $\beta$-smooth), and $|\ell'|\le G$ (i.e., $\ell$ is $G$-Lipschitz).
\end{assumption}
Under \Cref{ass:smooth}, the risk is also a smooth function of $W$, if all layers are bounded.
\begin{lemma}\label{lem:nn_smooth}
    Suppose $\ell$ is $\beta$-smooth. If $R\ge1$, then $\beta(R)=2L^2R^{2L-2}(\beta+G)$, and $\cR(W)$ is a $\beta(R)$-smooth function on the set $B(R)=\cbr{W\middle|\|W_k\|_F\le R,1\le k\le L}$.
\end{lemma}

Smoothness ensures that for any $W,V\in B(R)$, $\cR(W)-\cR(V)\le \langle\nR(V),W-V\rangle+\nicefrac{\beta(R)\|W-V\|^2}{2}$ (see \citet{bubeck_conv} Lemma 3.4). In particular, if we choose some $R$ and set a constant step size $\eta_t=1/\beta(R)$, then as long as $W(t+1)$ and $W(t)$ are both in $B(R)$,
\begin{align}
    \cR\del{W(t+1)}-\cR\del{W(t)} & \le \left\langle\nR\del{W(t)},-\eta_t\nR\del{W(t)}\right\rangle+\frac{\beta(R)\eta_t^2}{2}\enVert{\nR\del{W(t)}}^2 \nonumber \\
     & =-\frac{1}{2\beta(R)}\enVert{\nR\del{W(t)}}^2=-\frac{\eta_t}{2}\enVert{\nR\del{W(t)}}^2. \label{eq:gd_dec}
\end{align}
In other words, the risk does not increase at this step. However, similar to gradient flow, the gradient descent iterate will eventually escape $B(R)$, which may increase the risk.
\begin{lemma}\label{lem:gd_escape}
    Under \Cref{ass:loss}, \ref{ass:init} and \ref{ass:smooth}, suppose gradient descent is run with a constant step size $1/\beta(R)$. Then there exists a time $t$ when $W(t)\not\in B(R)$, in other words, $\max_{1\le k\le L}\|W_k(t)\|_F>R$.
\end{lemma}

Fortunately, this issue can be handled by adaptively increasing $R$ and correspondingly decreasing the step sizes, formalized in the following assumption.
\begin{assumption}\label{ass:step_size}
    The step size $\eta_t=\min\{1/\beta(R_t),1\}$, where $R_t$ satisfies $W(t)\in B(R_t-1)$, and if $W(t+1)\in B(R_t-1)$, $R_{t+1}=R_t$.
\end{assumption}
\Cref{ass:step_size} can be satisfied by a line search, which ensures that the gradient descent update is not too aggressive and the boundary $R$ is increased properly.

With the additional \Cref{ass:smooth,ass:step_size}, exactly the same theorems can be proved for gradient descent. We restate them briefly here.
\begin{theorem}\label{thm:gd_risk}
    Under \Cref{ass:loss}, \ref{ass:init}, \ref{ass:smooth}, and \ref{ass:step_size}, gradient descent satisfies
    \begin{itemize}
        \item $\lim_{t\to\infty}\cR\del{W(t)}=0$.
        \item For any $1\le k\le L$, $\lim_{t\to\infty}\|W_k(t)\|_F=\infty$.
        \item $\lim_{t\to\infty}\envert{\left\langle \nicefrac{\wnn(t)}{\prod_{k=1}^L\|W_k(t)\|_F},v_1(t)\right\rangle}=1$, where $v_1(t)$ is the first right singular vector of $W_1(t)$.
    \end{itemize}
\end{theorem}
\begin{theorem}\label{thm:gd_min_norm}
    Under \Cref{ass:init,ass:data,ass:step_size}, for the logistic loss $\llog$ and almost all data, $\lim_{t\to\infty}\envert{\langle v_1(t),\baru\rangle}=1$, and $\lim_{t\to\infty}\nicefrac{\wnn(t)}{\prod_{k=1}^L\|W_k(t)\|_F}=\baru$.
\end{theorem}

\begin{corollary}
  \label{fact:gd_min_norm_2}
    Under \Cref{ass:init,ass:data,ass:step_size}, for the logistic loss $\llog$ and almost all data,
    \[
      \min_i\ y_i \del{ \frac {W_L}{\|W_L\|_F} \cdots \frac {W_1}{\|W_1\|_F} } x_i
      \quad\xrightarrow[t\to\infty]{}\quad
      \max_{\substack{A_L \in \R^{1 \times d_{L-1}}\\\|A_L\|_F = 1}}
      \cdots
      \max_{\substack{A_1 \in \R^{d_1 \times d_0}\\\|A_1\|_F = 1}}
      \min_i\ y_i \del{ A_L \cdots A_1 } x_i.
    \]
\end{corollary}

Proofs of \Cref{thm:gd_risk} and \ref{thm:gd_min_norm} are given in \Cref{sec:gd_app}, and are basically the same as the gradient flow proofs. The key difference is that an error of $\sum_{t=0}^{\infty}\eta_t^2\|\nR(W(t))\|^2$ will be introduced in many parts of the proof. However, it is bounded in light of \cref{eq:gd_dec}:
\begin{equation*}
    \sum_{t=0}^{\infty}\eta_t^2\enVert{\nR\del{W(t)}}^2\le \sum_{t=0}^{\infty}\eta_t\enVert{\nR\del{W(t)}}^2\le2\cR\del{W(0)}.
\end{equation*}
Since all weight matrices go to infinity, such a bounded error does not matter asymptotically, and thus proofs still go through.

\section{Summary and future directions}\label{sec:future}

\begin{figure}[t]
    \centering
    \begin{subfigure}{0.495\textwidth}
        \centering
        \includegraphics[width=\textwidth]{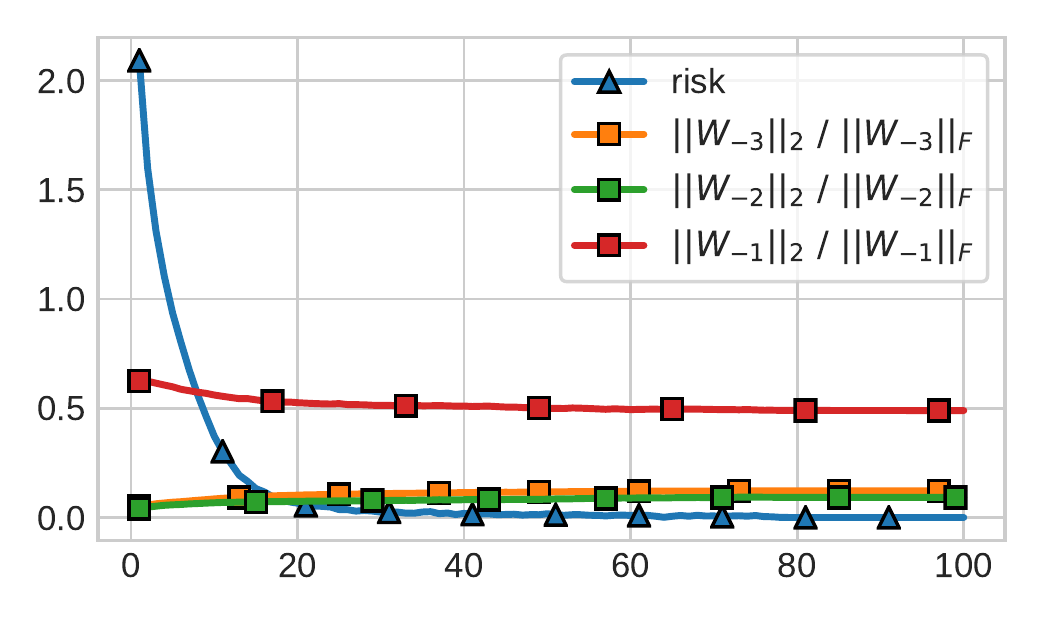}
        \caption{Default initialization.}
        \label{fig:cifar:vanilla}
    \end{subfigure}
    \begin{subfigure}{0.495\textwidth}
        \centering
        \includegraphics[width=\textwidth]{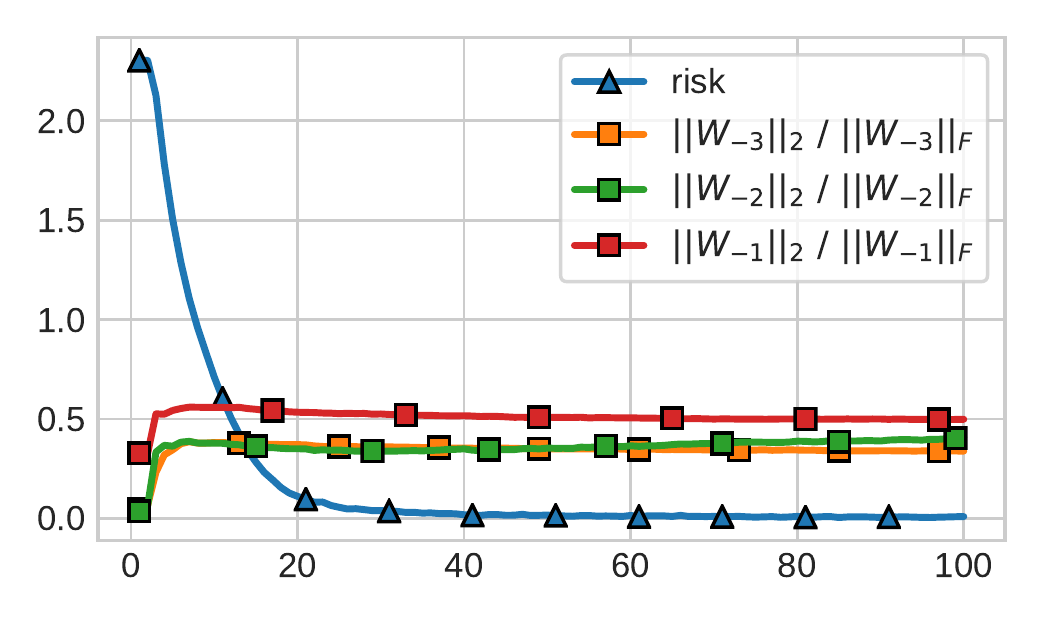}
        \caption{Initialization with the same Frobenius norm.}
        \label{fig:cifar:fancy}
    \end{subfigure}
    \caption{Risk and alignment of dense layers (the ratio $\|W_i\|_2/\|W_i\|_F$) of (nonlinear!) AlexNet on CIFAR-10.
    \Cref{fig:cifar:vanilla} uses default PyTorch initialization,
    while \Cref{fig:cifar:fancy} forces initial Frobenius norms to be equal amongst dense layers.}
    \label{fig:cifar}
\end{figure}

This paper rigorously proves that, for deep linear networks on linearly separable data, gradient flow and gradient descent minimize the risk to $0$, align adjacent weight matrices, and align the first right singular vector of the first layer to the maximum margin solution determined by the data.
There are many potential future directions; a few are as follows.

\paragraph{Convergence rates and practical step sizes.} This paper only proves asymptotic convergence, moreover for adaptive step sizes depending on the current weight matrix norms.  A refined analysis with rates for practical step sizes (e.g., constant step sizes) would allow the algorithm to be compared to other methods which also globally optimize this objective,
would suggest ways to improve step sizes and initialization,
and ideally even exhibit a sensitivity to the network architecture and suggest how it could be improved.

\paragraph{Nonseparable data and nonlinear networks.}
Real-world data is generally not linearly separable,
but nonlinear deep networks can reliably decrease the risk to 0,
even with random labels \citep{rethinking}.
This seems to suggest that a nonlinear notion of separability is at play;
is there some way to adapt the present analysis?

The present analysis is crucially tied to the alignment of weight matrices:
alignment and risk are analyzed simultaneously.  Motivated by this, consider
a preliminary experiment, presented in \Cref{fig:cifar},
where stochastic gradient descent was used to minimize the risk
of a standard AlexNet on CIFAR-10 \citep{alexnet,cifar}.

Even though there are ReLUs, max-pooling layers, and convolutional layers,
the alignment phenomenon is occurring in a reduced form on the dense layers
(the last three layers of the network).  Specifically, despite these weight matrices
having shape $(1024,4096)$, $(4096,4096)$, and $(4096,10)$ the key alignment ratios
$\|W_i\|_2/\|W_i\|_F$
are much larger than their respective lower bounds $(1024^{-1/2}, 4096^{-1/2}, 10^{-1/2})$.
Two initializations were tried: default PyTorch initialization, and a Gaussian initialization
forcing all initial Frobenius norms to be just $4$, which is suggested by the norm preservation
property in the analysis and removes noise in the weights.

\subsection*{Acknowledgements}

The authors are grateful for support from the NSF under grant IIS-1750051.
This grant allowed them to focus on research,
and when combined with an NVIDIA GPU grant,
led to the creation of their beloved GPU machine $\textsc{DutchCrunch}$.

\bibliography{bib}
\bibliographystyle{plainnat}

\clearpage

\appendix

\section{Regarding \Cref{ass:init}}\label{sec:notation_app}

Suppose $W_1(0)=0$ while $W_L(0)\cdots W_2(0)\ne0$. First of all, $W_L(0)\cdots W_1(0)=0$ and thus $\cR\del{W(0)}=\cR(0)$. Moreover,
\begin{equation*}
    \left\langle\nR\del{\wnn(0)},\baru\right\rangle=\frac{1}{n}\sum_{i=1}^{n}\ell'(0)\langle z_i,\baru\rangle\le\ell'(0)\gamma<0,
\end{equation*}
which implies $\nR\del{\wnn(0)}\ne0$ and $\partial\cR/\partial W_1=\del{W_L(0)\cdots W_2(0)}^{\top}\nR\del{\wnn(0)}^{\top}\ne0$.

On the other hand, if $\cR\del{W(0)}>\cR(0)$, gradient flow/descent may never minimize the risk to $0$. For example, suppose the network has two layers, and both the input and output have dimension $1$; the network just computes the dot product of two vectors $w_1$ and $w_2$. Consider minimizing $\cR(w_1,w_2)=\exp\del{-\langle w_1,w_2\rangle}$. If $w_1(0)=-w_2(0)\ne0$, then $\cR\del{w_1(0),w_2(0)}=\exp\del{\|w_1\|^2}>\exp(0)$. It is easy to verify that for any $t$, $w_1(t)=-w_2(t)$, and $\cR\del{w_1(t),w_2(t)}\ge\exp(0)>0$.

\section{Omitted proofs from \Cref{sec:gf}}\label{sec:gf_app}

\begin{proof}[Proof of \Cref{lem:gf_unbounded}]
    Fix an arbitrary $R>0$. If the claim is not true, then for any $\epsilon>0$, there exists some $t\ge1$ such that $\|W_k\|_F\le R$ for all $k$ while $\enVert{\partial\cR/\partial W_1}_F^2\le\epsilon^2$, which means
    \begin{equation*}
        \enVert{\frac{\partial\cR}{\partial W_1}}_F^2=\enVert{W_2^{\top}\cdots W_L^{\top}\nabla\cR(\wnn)^{\top}}_F^2=\enVert{W_L\cdots W_2}^2\enVert{\nabla\cR(\wnn)}^2\le\epsilon^2.
    \end{equation*}
    Since $\|\wnn\|\le R^L$, we have
    \begin{equation*}
        \langle\nR(\wnn),\baru\rangle=\frac{1}{n}\sum_{i=1}^{n}\ell'\del{\langle\wnn,z_i\rangle}\langle z_i,\baru\rangle\le \frac{1}{n}\sum_{i=1}^{n}\ell'\del{\langle\wnn,z_i\rangle}\gamma\le-M\gamma,
    \end{equation*}
    where $-M=\max_{-R^L\le x\le R^L}\ell'(x)$. Since $\ell'$ is continuous and the domain is bounded, the maximum is attained and negative, and thus $M>0$. Therefore $\enVert{\nabla\cR(\wnn)}\ge M\gamma$, and thus $\|W_L\cdots W_2\|\le\epsilon/M\gamma$. Since $\|W_1\|_F\le R$, we further have $\|\wnn\|\le\epsilon R/M\gamma$. In other words,  after $t=1$, $\enVert{\wnn}$ may be arbitrarily small, which implies $\cR\del{\wnn}$ can be arbitrarily close to $\cR\del{0}$.

    On the other hand, by \Cref{ass:init}, $\dif\cR(W)/\dif t=-\|\nR(W)\|^2<0$ at $t=0$. This implies that $\cR\del{W(1)}<\cR\del{W(0)}$, and for any $t\ge1$, $\cR\del{W(t)}\le\cR\del{W(1)}<\cR\del{W(0)}\le\cR(0)$, which is a contradiction.

    Since the risk is always positive, we have
    \begin{align*}
        \cR\del{W(0)} & \ge\int_{t=0}^{\infty}\sum_{k=1}^{L}\enVert{\frac{\partial\cR}{\partial W_k}}_F^2\dif t \\
         & \ge\int_{t=0}^{\infty}\enVert{\frac{\partial\cR}{\partial W_1}}_F^2\dif t \\
         & \ge\int_{t=0}^{\infty}\enVert{\frac{\partial\cR}{\partial W_1}}_F^2\mathds{1}\sbr{\max_{1\le k\le L}\|W_k\|_F\le R}\dif t \\
         & \ge\int_{t=1}^{\infty}\enVert{\frac{\partial\cR}{\partial W_1}}_F^2\mathds{1}\sbr{\max_{1\le k\le L}\|W_k\|_F\le R}\dif t \\
         & \ge\epsilon(R)^2\int_{t=1}^{\infty}\mathds{1}\sbr{\max_{1\le k\le L}\|W_k\|_F\le R}\dif t,
    \end{align*}
    which implies gradient flow only spends a finite amount of time in $\cbr{W\middle|\max_{1\le k\le L}\|W_k\|_F\le R}$. This directly implies that $\max_{1\le k\le L}\|W_k\|_F$ is unbounded.
\end{proof}

\begin{proof}[Proof of \Cref{lem:gf_align}]
    The first claim is true for $k=L$ since $W_L$ is a row vector. For any $1\le k<L$, recall that \citet{arora_icml,jason_nips} give the following relation:
    \begin{equation}\label{eq:adj_layers}
        W_{k+1}^{\top}(t)W_{k+1}(t)-W_{k+1}^{\top}(0)W_{k+1}(0)=W_k(t)W_k^{\top}(t)-W_k(0)W_k^{\top}(0).
    \end{equation}
    Let $A_{k,k+1}=W_{k}(0)W_k^{\top}(0)-W_{k+1}^{\top}(0)W_{k+1}(0)$. By \cref{eq:adj_layers} and the definition of singular vectors and singular values, we have
    \begin{align}
        \sigma_k^2 & \ge v_{k+1}^{\top}W_kW_k^{\top}v_{k+1} \nonumber \\
         & =v_{k+1}^{\top}W_{k+1}^{\top}W_{k+1}v_{k+1}+v_{k+1}^{\top}A_{k,k+1}v_{k+1} \nonumber \\
         & =\sigma_{k+1}^2+v_{k+1}^{\top}A_{k,k+1}v_{k+1} \nonumber \\
         & \ge\sigma_{k+1}^2-\|A_{k,k+1}\|_2. \label{eq:svd_1}
    \end{align}
    Moreover, by taking the trace on both sides of \cref{eq:adj_layers}, we have
    \begin{align}
        \|W_k\|_F^2=\tr\del{W_kW_k^{\top}} & =\tr\del{W_{k+1}^{\top}W_{k+1}}+\tr\del{W_k(0)W_k^{\top}(0)}-\tr\del{W_{k+1}^{\top}(0)W_{k+1}(0)} \nonumber \\
         & =\|W_{k+1}\|_F^2+\|W_k(0)\|_F^2-\|W_{k+1}(0)\|_F^2. \label{eq:ind_2}
    \end{align}
    Summing \cref{eq:svd_1} and \cref{eq:ind_2} from $k$ to $L-1$, we get
    \begin{equation}\label{eq:partial_sum}
        \|W_k\|_F^2-\|W_k\|_2^2\le\|W_k(0)\|_F^2-\|W_L(0)\|_F^2+\sum_{k'=k}^{L-1}\|A_{k',k'+1}\|_2\le D.
    \end{equation}

    Next we prove that singular vectors get aligned. Consider $u_k^{\top}W_{k+1}^{\top}W_{k+1}u_k$. On one hand, similar to \cref{eq:svd_1}, we can get that
    \begin{align}\label{eq:svd_2}
        u_k^{\top}W_{k+1}^{\top}W_{k+1}u_k & =u_k^{\top}W_kW_k^{\top}u_k-u_k^{\top}W_k(0)W_k^{\top}(0)u_k+u_k^{\top}W_{k+1}^{\top}(0)W_{k+1}(0)u_k \nonumber \\
         & \ge u_k^{\top}W_kW_k^{\top}u_k-u_k^{\top}W_k(0)W_k^{\top}(0)u_k \nonumber \\
         & \ge\sigma_k^2-\|W_k(0)\|_2^2.
    \end{align}
    On the other hand, it follows from the definition of singular vectors and \cref{eq:partial_sum} that
    \begin{align}
        u_k^{\top}W_{k+1}^{\top}W_{k+1}u_k & =\langle u_k,v_{k+1}\rangle^2\sigma_{k+1}^2+u_k^{\top}\del{W_{k+1}^{\top}W_{k+1}-v_{k+1}\sigma_{k+1}^2v_{k+1}^{\top}}u_k \nonumber \\
         & \le \langle u_k,v_{k+1}\rangle^2\sigma_{k+1}^2+\|W_{k+1}\|_F^2-\|W_{k+1}\|_2^2 \nonumber \\
         & \le \langle u_k,v_{k+1}\rangle^2\sigma_{k+1}^2+D. \label{eq:svd_3}
    \end{align}
    Combining \cref{eq:svd_2} and \cref{eq:svd_3}, we get
    \begin{align}
        \sigma_k^2 & \le \langle u_k,v_{k+1}\rangle^2\sigma_{k+1}^2+D+\|W_k(0)\|_2^2. \label{eq:dp_prelim}
    \end{align}
    Similar to \cref{eq:svd_2}, we can get
    \begin{align*}
        \sigma_k^2\ge v_{k+1}^{\top}W_kW_k^{\top}v_{k+1}\ge\sigma_{k+1}^2-\|W_{k+1}(0)\|_2^2.
    \end{align*}
    Therefore
    \begin{equation}\label{eq:sv_ratio}
        \frac{\sigma_k^2}{\sigma_{k+1}^2}\ge1-\frac{\|W_{k+1}(0)\|_2^2}{\sigma_{k+1}^2}.
    \end{equation}
    Combining \cref{eq:dp_prelim} and \cref{eq:sv_ratio}, we finally get
    \begin{equation*}
        \langle u_k,v_{k+1}\rangle^2\ge1-\frac{D+\|W_k(0)\|_2^2+\|W_{k+1}(0)\|_2^2}{\sigma_{k+1}^2}.
    \end{equation*}

    Regarding the last claim, first recall that since the difference between the squares of Frobenius norms of any two layers is a constant, $\max_{1\le k\le L}\|W_k\|_F\to\infty$ implies $\|W_k\|_F\to\infty$ for any $k$. We further have the following.
    \begin{itemize}
        \item Since $\|W_k\|_F^2-\|W_k\|_2^2\le D$, $\|W_k\|_2\to\infty$ for any $k$, and $W_k/\|W_k\|_F\to u_kv_k^{\top}$.
        \item Since $\|W_k\|_2\to\infty$, $|\langle u_k,v_{k+1}\rangle|\to1$.
    \end{itemize}
     As a result,
     \begin{align*}
         \envert{\left\langle \frac{\wnn}{\prod_{k=1}^L\|W_k\|_F},v_1\right\rangle} & =\envert{\left\langle \prod_{k=1}^{L}\frac{W_k}{\|W_k\|_F},v_1\right\rangle} \\
          & \to\envert{\left\langle \prod_{k=1}^{L}u_iv_i^{\top},v_1\right\rangle} \\
          & \to1.
     \end{align*}
\end{proof}

\begin{proof}[Proof of \Cref{thm:gf_risk}]
    Suppose for some $\epsilon>0$, $\cR\del{W}\ge\epsilon$ for any $t$. Then there exists some $1\le j\le n$ such that $\ell\del{\langle\wnn,z_j\rangle}\ge\epsilon$, and thus $\langle\wnn,z_j\rangle\le\ell^{-1}(\epsilon)$. On the other hand, since $\cR(W)\le\cR(0)=\ell(0)$, $\ell\del{\langle\wnn,z_j\rangle}\le n\ell(0)$, and thus $\langle\wnn,z_j\rangle\ge\ell^{-1}\del{n\ell(0)}$. Let $-M=\max_{\ell^{-1}\del{n\ell(0)}\le x\le \ell^{-1}(\epsilon/n)}\ell'(x)<0$, we have for any $t$,
    \begin{align*}
        \langle\nR(\wnn),\baru\rangle & =\frac{1}{n}\sum_{i=1}^{n}\ell'\del{\langle\wnn,z_i\rangle}\langle z_i,\baru\rangle \\
         & \le \frac{1}{n}\sum_{i=1}^{n}\ell'\del{\langle\wnn,z_i\rangle}\gamma \\
         & \le \frac{1}{n}\ell'\del{\langle\wnn,z_j\rangle}\gamma \\
         & \le \frac{-M\gamma}{n}<0,
    \end{align*}
    and thus $\|\nR(\wnn)\|\ge M\gamma/n$.

    Similar to the proof of \Cref{lem:gf_align}, we can show that if $\|W_k\|_F\to\infty$,
    \begin{equation*}
        \envert{\left\langle \frac{(W_L\cdots W_2)^{\top}}{\|W_k\|_F\cdots\|W_2\|_F},v_2\right\rangle}\to1.
    \end{equation*}
    In other words, there exists some $C>0$, such that when $\min_{1\le k\le L}\|W_k\|_F>C$, $\|W_L\cdots W_2\|\ge\|W_k\|_F\cdots\|W_2\|_F/2>C^L/2$.

    \Cref{lem:gf_unbounded} shows that gradient flow spends a finite amount of time in $\cbr{W\middle|\max_{1\le k\le L}\|W_k\|_F\le R}$ for any $R>0$. Since the difference between the squares of Frobenius norms of any two layers is a constant, gradient flow also spends a finite amount of time in $\cbr{W\middle|\min_{1\le k\le L}\|W_k\|_F\le C}$. Now we have
    \begin{align*}
        \cR\del{W(0)} & \ge\int_{t=0}^{\infty}\sum_{k=1}^{L}\enVert{\frac{\partial\cR}{\partial W_k}}_F^2\dif t \\
         & \ge\int_{t=0}^{\infty}\enVert{\frac{\partial\cR}{\partial W_1}}_F^2\dif t \\
         & =\int_{t=0}^{\infty}\|W_L\cdots W_2\|^2\|\nR(\wnn)\|^2\dif t \\
         & \ge\int_{t=0}^{\infty}\|W_L\cdots W_2\|^2\|\nR(\wnn)\|^2\mathds{1}\sbr{W\middle|\min_{1\le k\le L}\|W_k\|_F>C}\dif t \\
         & \ge\del{\frac{M\gamma}{n}}^2\del{\frac{C^L}{2}}^2\int_{t=0}^{\infty}\mathds{1}\sbr{W\middle|\min_{1\le k\le L}\|W_k\|_F>C}\dif t \\
         & =\infty,
    \end{align*}
    which is a contradiction. Therefore $\cR(\epsilon)\to0$. This further implies $\|W_k\|_F\to\infty$, since $\cR(W)$ has no finite optimum. Finally, invoking \Cref{lem:gf_align} proves the final claim of \Cref{thm:gf_risk}.
\end{proof}

\begin{proof}[Proof of \Cref{lem:strong_conv}]
    \cite{nati_iclr} Lemma 12 proves that, with probability $1$, there are at most $d$ support vectors, and moreover, the $i$-th support vector $z_i$ has a positive dual variable $\alpha_i$, such that $\sum_{i\in S}^{}\alpha_iz_i=\baru$.

    Suppose there exists some $\xi\perp\baru$, such that $\max_{i\in S}\langle\xi,z_i\rangle\le0$. Since
    \begin{equation*}
        \sum_{i\in S}^{}\alpha_i \langle\xi,z_i\rangle=\left\langle\xi,\sum_{i\in S}^{}\alpha_iz_i\right\rangle=\langle\xi,\baru\rangle=0,
    \end{equation*}
    we actually have $\langle\xi,z_i\rangle=0$ for all $i\in S$. This is impossible under \Cref{ass:data}, since the support vectors span the whole space.
\end{proof}

\begin{proof}[Proof of \Cref{lem:perp_bound}]
    For the sake of presentation, we leave out the subscript in $z_i$ and denote a data point by $z$ generally. For any data point $z$ and predictor $w$, let $z_{\perp}$ and $w_{\perp}$ denote their projection onto $\baru^{\perp}$. Let $z'\in\arg\max_{i\in S}\langle -w_{\perp},z\rangle$, and thus $\langle -w_{\perp},z'\rangle\ge\alpha\|w_{\perp}\|$.

    For $\lexp$, we have
    \begin{align}
        \langle w_{\perp},\nabla\cR(w)\rangle & =\sum_{z}^{}\frac{1}{n}\sbr{-\exp\del{-\langle w,z\rangle}}\langle w_{\perp},z_{\perp}\rangle \nonumber \\
         & =\sum_{z}^{}\frac{1}{n}\sbr{\exp\del{\langle-w,z\rangle}}\langle-w_{\perp},z_{\perp}\rangle \nonumber \\
         & \ge\frac{1}{n}\exp\del{\langle-w,z'\rangle}\langle-w_{\perp},z'_{\perp}\rangle+\sum_{\langle z_{\perp},w_{\perp}\rangle\ge0}^{}\frac{1}{n}\exp\del{\langle-w,z\rangle}\langle-w_{\perp},z_{\perp}\rangle. \label{eq:perp_bound_tmp1}
    \end{align}
    The first part can be lower bounded as below (recall that $\langle -w_{\perp},z'_{\perp}\rangle=\langle -w_{\perp},z'\rangle\ge\alpha\|w_{\perp}\|$)
    \begin{align}
        \frac{1}{n}\exp\del{\langle-w,z'\rangle}\langle-w_{\perp},z'_{\perp}\rangle & =\frac{1}{n}\exp\del{\langle -w,\gamma\bar{u}\rangle}\exp\del{\langle-w_{\perp},z'_{\perp}\rangle}\langle-w_{\perp},z'_{\perp}\rangle \nonumber \\
         & \ge \frac{1}{n}\exp\del{-\langle w,\gamma\bar{u}\rangle}\exp\del{\alpha\|w_{\perp}\|}\alpha\|w_{\perp}\|. \label{eq:perp_bound_tmp2}
    \end{align}

    To bound the second part, first notice that since we assume $\langle w,\bar{u}\rangle\ge0$, for any $z$,
    \begin{equation}\label{eq:perp_bound_tmp3}
        \langle w,z-\gamma\baru\rangle=\langle w,z_{\perp}\rangle+\langle w,z-\gamma\baru-z_{\perp}\rangle\ge \langle w,z_{\perp}\rangle=\langle w_{\perp},z_{\perp}\rangle.
    \end{equation}
    The reason is that every data point has margin at least $\gamma$, and thus $z-\gamma\baru-z_{\perp}=c\baru$ for some $c\ge0$. Using \cref{eq:perp_bound_tmp3}, we can bound the second part of \cref{eq:perp_bound_tmp1}.
    \begin{align}
         & \sum_{\langle z_{\perp},w_{\perp}\rangle\ge0}^{}\frac{1}{n}\exp\del{\langle-w,z\rangle}\langle-w_{\perp},z_{\perp}\rangle \nonumber \\
        = & \sum_{\langle z_{\perp},w_{\perp}\rangle\ge0}^{}\frac{1}{n}\exp\del{\langle-w,\gamma\bar{u}\rangle}\exp\del{\langle -w,z-\gamma\bar{u}\rangle}\langle-w_{\perp},z_{\perp}\rangle \nonumber \\
        \ge & \sum_{\langle z_{\perp},w_{\perp}\rangle\ge0}^{}\frac{1}{n}\exp\del{\langle-w,\gamma\bar{u}\rangle}\exp\del{\langle -w_{\perp},z_{\perp}}\langle-w_{\perp},z_{\perp}\rangle \nonumber \\
        \ge & \sum_{\langle z_{\perp},w_{\perp}\rangle\ge0}^{}\frac{1}{n}\exp\del{\langle-w,\gamma\bar{u}\rangle}\del{-\frac{1}{e}} \nonumber \\
        \ge & \quad\exp\del{\langle-w,\gamma\bar{u}\rangle}\del{-\frac{1}{e}}. \label{eq:perp_bound_tmp4}
    \end{align}
    On the third line \cref{eq:perp_bound_tmp3} is applied. The fourth line applies the property that $f(x)=-xe^{-x}\ge-1/e$ when $x\ge0$.

    Combining \cref{eq:perp_bound_tmp1}, \cref{eq:perp_bound_tmp2} and \cref{eq:perp_bound_tmp4}, we get
    \begin{equation*}
        \langle w_{\perp},\nR(w)\rangle\ge\exp\del{\langle-w,\gamma\bar{u}\rangle}\del{\frac{1}{n}\exp\del{\alpha\|w_{\perp}\|}\alpha\|w_{\perp}\|-\frac{1}{e}}.
    \end{equation*}
    As long as $\|w_{\perp}\|\ge(1+\ln(n))/\alpha$, $\langle w_{\perp},\nR(w)\rangle\ge0$.

    For $\llog$, similar to \cref{eq:perp_bound_tmp1}, we have
    \begin{align}
        \langle w_{\perp},\nabla\cR(w)\rangle & \ge \frac{1}{n}\frac{\exp\del{\langle-w,z'\rangle}}{1+\exp\del{\langle-w,z'\rangle}}\langle-w_{\perp},z'_{\perp}\rangle+\sum_{\langle z_{\perp},w_{\perp}\rangle\ge0}^{}\frac{1}{n}\frac{\exp\del{\langle-w,z\rangle}}{1+\exp\del{\langle-w,z\rangle}}\langle-w_{\perp},z_{\perp}\rangle \nonumber\\
        & \ge \frac{1}{n}\frac{\exp\del{\langle-w,z'\rangle}}{1+\exp\del{\langle-w,z'\rangle}}\langle-w_{\perp},z'_{\perp}\rangle+\sum_{\langle z_{\perp},w_{\perp}\rangle\ge0}^{}\frac{1}{n}\exp\del{\langle-w,z\rangle}\langle-w_{\perp},z_{\perp}\rangle. \label{eq:perp_bound_tmp5}
    \end{align}
    The second part of \cref{eq:perp_bound_tmp5} can be bounded again by \cref{eq:perp_bound_tmp4}. To bound the first part of \cref{eq:perp_bound_tmp5}, first notice that (recall $\langle w,\baru\rangle\ge0$)
    \begin{equation}\label{eq:perp_bound_tmp6}
        \exp\del{\langle-w,z'\rangle}=\exp\del{\langle -w,\gamma\bar{u}\rangle}\exp\del{\langle-w_{\perp},z'_{\perp}\rangle}\le\exp\del{\langle-w_{\perp},z'_{\perp}\rangle}.
    \end{equation}
    Using \cref{eq:perp_bound_tmp6}, and recall that $\langle -w_{\perp},z'_{\perp}\rangle=\langle -w_{\perp},z'\rangle\ge\alpha\|w_{\perp}\|\ge0$, we can bound the first part of \cref{eq:perp_bound_tmp5} as below.
    \begin{align}
        \frac{1}{n}\frac{\exp\del{\langle-w,z'\rangle}}{1+\exp\del{\langle-w,z'\rangle}}\langle-w_{\perp},z'_{\perp}\rangle & =\frac{1}{n}\exp\del{\langle -w,\gamma\bar{u}\rangle}\frac{\exp\del{\langle-w_{\perp},z'_{\perp}\rangle}}{1+\exp\del{\langle-w,z'\rangle}}\langle-w_{\perp},z'_{\perp}\rangle \nonumber \\
         & \ge \frac{1}{n}\exp\del{\langle -w,\gamma\bar{u}\rangle}\frac{\exp\del{\langle-w_{\perp},z'_{\perp}\rangle}}{1+\exp\del{\langle-w_{\perp},z'_{\perp}\rangle}}\langle-w_{\perp},z'_{\perp}\rangle \nonumber \\
         & \ge \frac{1}{2n}\exp\del{\langle -w,\gamma\bar{u}\rangle}\langle-w_{\perp},z'_{\perp}\rangle \nonumber \\
         & \ge \frac{1}{2n}\exp\del{\langle -w,\gamma\bar{u}\rangle}\alpha\|w_{\perp}\|. \label{eq:perp_bound_tmp7}
    \end{align}
    Combining \cref{eq:perp_bound_tmp5}, \cref{eq:perp_bound_tmp7} and \cref{eq:perp_bound_tmp4}, we get
    \begin{equation*}
        \langle w_{\perp},\nR(w)\rangle\ge\exp\del{\langle-w,\gamma\bar{u}\rangle}\del{\frac{1}{2n}\alpha\|w_{\perp}\|-\frac{1}{e}}.
    \end{equation*}
    As long as $\|w_{\perp}\|\ge2n/e\alpha$, $\langle w_{\perp},\nR(w)\rangle\ge0$.
\end{proof}

\begin{proof}[Proof of \Cref{thm:gf_min_norm}]
    Recall that
    \begin{equation*}
        \frac{\dif W_1}{\dif t}=-\frac{\partial\cR}{\partial W_1}=-W_2^{\top}\cdots W_L^{\top}\nR(\wnn)^{\top},
    \end{equation*}
    and thus
    \begin{equation}\label{eq:w1_speed}
        \frac{\dif\|W_1\|_F^2}{\dif t}=\left\langle W_1,\frac{\dif W_1}{\dif t}\right\rangle=-2\langle\wnn,\nR(\wnn)\rangle.
    \end{equation}

    Let $\Pi_{\baru}$ denote the projection onto $\mathrm{span}(\baru)$, and let $\Pi_{\perp}$ denote the projection onto $\baru^{\perp}$. Also let $\Pi_{\baru}W_1$ and $\Pi_{\perp}W_1$ denote the projection of rows of $W_1$ onto $\mathrm{span}(\baru)$ and $\baru^{\perp}$, respectively. Notice that
    \begin{equation*}
        \Pi_{\baru}\wnn=\del{W_L\cdots W_2(\Pi_{\baru}W_1)}^{\top},\quad\textrm{and}\quad\Pi_{\perp}\wnn=\del{W_L\cdots W_2(\Pi_{\perp}W_1)}^{\top}.
    \end{equation*}
    We further have
    \begin{equation}\label{eq:perp_speed}
        \frac{\dif\|\Pi_{\perp}W_1\|_F^2}{\dif t}=-2\langle\Pi_{\perp}\wnn,\nR(\wnn)\rangle.
    \end{equation}

    Let $W_1=u_1\sigma_1v_1^{\top}+S$. We have $\|S\|_2\le\sigma_{1,2}\le\sqrt{\sigma_{1,2}^2}\le\sqrt{\|W_1\|_F^2-\|W_1\|_2^2}\le\sqrt{D}$, where $\sigma_{1,2}$ is the second singular value of $W_1$ and $D$ is the constant introduced in \Cref{lem:gf_align}. Then
    \begin{equation*}
        \Pi_{\perp}W_1=u_1\sigma_1\del{\Pi_{\perp}v_1}^{\top}+\Pi_{\perp}S,
    \end{equation*}
    and
    \begin{equation*}
        \|\Pi_{\perp}W_1\|_F\le\enVert{u_1\sigma_1\del{\Pi_{\perp}v_1}^{\top}}_F+\|\Pi_{\perp}S\|_F=\sigma_1\|\Pi_{\perp}v_1\|+\|\Pi_{\perp}S\|_F\le\sigma_1\|\Pi_{\perp}v_1\|+\sqrt{dD}.
    \end{equation*}
    It follows that
    \begin{equation}\label{eq:min_norm_tmp1}
        \|\Pi_{\perp}v_1\|\ge \frac{\|\Pi_{\perp}W_1\|_F}{\sigma_1}-\frac{\sqrt{dD}}{\sigma_1}\ge \frac{\|\Pi_{\perp}W_1\|_F}{\|W_1\|_F}-\frac{\sqrt{dD}}{\|W_1\|_2}.
    \end{equation}

    Fix an arbitrary $\epsilon>0$. By \Cref{thm:gf_risk}, we can find some $t_0$ large enough such that for any $t\ge t_0$:
    \begin{enumerate}
        \item $\sqrt{dD}/\|W_1\|_2\le\epsilon/3$.
        \item $\|\nicefrac{\wnn}{\|W_L\|_F\cdots\|W_1\|_F}-v_1\|\le\epsilon/3$, or $\|\nicefrac{\wnn}{\|W_L\|_F\cdots\|W_1\|_F}+v_1\|\le\epsilon/3$.
        \item $\|W_L\|_F\cdots\|W_1\|_F\ge 3K/\epsilon$, where $K$ is the threshold given in \Cref{lem:perp_bound}, i.e., $\nicefrac{1+\ln(n)}{\alpha}$ for $\lexp$, $\nicefrac{2n}{e\alpha}$ for $\llog$.
        \item $\cR(W)\le\ell(0)/n$, which implies $\langle\wnn,z_i\rangle\ge0$ for all $1\le i\le n$. By \Cref{lem:strong_conv}, there always exists a support vector $z$ for which $\langle\Pi_{\perp}\wnn,z\rangle\le0$, and therefore $\langle\wnn,\baru\rangle\ge0$.
    \end{enumerate}

    Suppose for some $t\ge t_0$, $\|\Pi_{\perp}W_1\|_F/\|W_1\|_F\ge\epsilon$. By \cref{eq:min_norm_tmp1} and bullet $1$ above, $\|\Pi_{\perp}v_1\|\ge2\epsilon/3$. Bullet $2$ above then gives $\|\nicefrac{\Pi_{\perp}\wnn}{\|W_L\|_F\cdots\|W_1\|_F}\|\ge\epsilon/3$, which together with bullet $3$ above implies $\|\Pi_{\perp}\wnn\|\ge K$. Since also $\langle\wnn,\baru\rangle\ge0$, we can apply \Cref{lem:perp_bound} and get that $\langle\Pi_{\perp}\wnn,\nR(\wnn)\rangle\ge0$. In light of \cref{eq:perp_speed}, $\dif\|\Pi_{\perp}W_1\|_F^2/\dif t\le0$.

    On the other hand, since after $t\ge t_0$, $\langle\wnn,z_i\rangle\ge0$, we have $\dif\|W_1\|_F^2/\dif t\ge0$ by \cref{eq:w1_speed}. Therefore $\|\Pi_{\perp}W_1\|_F/\|W_1\|_F$ will not increase, and since $\|W_1\|_F\to\infty$, it will eventually drop below $\epsilon$, and will never exceed $\epsilon$ again. Therefore,
    \begin{equation*}
        \lim\sup_{t\to\infty}\frac{\|\Pi_{\perp}W_1\|_F}{\|W_1\|_F}\le\epsilon.
    \end{equation*}
    Since $\epsilon$ is arbitrary, we have
    \begin{equation*}
        \lim\sup_{t\to\infty}\frac{\|\Pi_{\perp}W_1\|_F}{\|W_1\|_F}=0,
    \end{equation*}
    and thus $\lim_{t\to\infty}\envert{\langle v_1,\baru\rangle}=1$. An application of \Cref{thm:gf_risk} gives the other part of \Cref{thm:gf_min_norm}.
\end{proof}

\begin{proof}[Proof of \Cref{fact:gf_min_norm_2}]
  By \Cref{thm:gf_min_norm},
  \[
    \min_i y_i \del{ \frac {W_L}{\|W_L\|_F} \cdots \frac {W_1}{\|W_1\|_F} } x_i
    \quad\longrightarrow\quad
    \min_i y_i \baru^\top x_i
   .
  \]
  Next, pick arbitrary unit vectors $a_i \in \R^{d_i}$, and note
  \begin{align*}
    \max_{\substack{A_1 \in \R^{d_1 \times d_0}\\\|A_1\|_F = 1}}
      \cdots
      \max_{\substack{A_L \in \R^{1 \times d_{L-1}}\\\|A_L\|_F = 1}}
    \min_i y_i \del{ A_L \cdots A_1 } x_i
    &\geq
    \min_i y_i \del{ (1 a_{L-1}^\top) (a_{L-1} a_{L-2}^\top) \cdots (a_2a_1^\top)(a_1\baru^\top) } x_i
    \\
    &=
    \min_i y_i \baru^\top x_i.
  \end{align*}
  On the other hand, if matrices $(A_L,\ldots,A_1)$ are feasible, then
  \[
    \|A_L \cdots A_1\|_2 \leq \|A_L\|_2 \cdots \|A_2\|_2 \|A_1\|_F \leq \|A_L\|_F \cdots \|A_1\|_F \leq 1,
  \]
  whereby
  \begin{align*}
      \min_i y_i (A_L\cdots A_1) x_i & \leq \|A_L \cdots A_1\|_2\min_i y_i \del{ \frac {A_L \cdots A_1}{\|A_L \cdots A_1\|_2} }  x_i \\
       & \leq 1 \cdot \max_{\|w\|_2 = 1} \min_i y_i w^\top x_i \\
       & =\min_i y_i \baru^\top x_i.
  \end{align*}
\end{proof}

\section{Omitted proofs from \Cref{sec:gd}}\label{sec:gd_app}

\begin{proof}[Proof of \Cref{lem:nn_smooth}]
    Given $W,V\in B(R)$, we need to show that $\|\nR(W)-\nR(V)\|\le\beta(R)\|W-V\|$ for some $\beta(R)$.

    Consider $k=1$ first. Let $w=(W_L\cdots W_1)^{\top}$, and $v=(V_L\cdots V_1)^{\top}$. Since $|\ell'|\le G$, $\|\nR(w)\|,\|\nR(v)\|\le G$. We have
    \begin{align}
        \enVert{\frac{\partial\cR}{\partial W_1}-\frac{\partial\cR}{\partial V_1}} & =\enVert{W_2^{\top}\cdots W_L^{\top}\nR(w)-V_2^{\top}\cdots V_L^{\top}\nR(v)} \nonumber \\
         & \le\enVert{W_2^{\top}\cdots W_L^{\top}\nR(w)-V_2^{\top}W_3^{\top}\cdots W_L^{\top}\nR(w)} \nonumber \\
         & \quad+\enVert{V_2^{\top}W_3^{\top}\cdots W_L^{\top}\nR(w)-V_2^{\top}\cdots V_L^{\top}\nR(v)} \nonumber \\
         & \le R^{L-2}G\|W_2-V_2\| \nonumber \\
         & \quad+\enVert{V_2^{\top}W_3^{\top}\cdots W_L^{\top}\nR(w)-V_2^{\top}\cdots V_L^{\top}\nR(v)} \nonumber \\
         & \le R^{L-2}G\|W-V\| \nonumber \\
         & \quad+\enVert{V_2^{\top}W_3^{\top}\cdots W_L^{\top}\nR(w)-V_2^{\top}\cdots V_L^{\top}\nR(v)}. \label{eq:smooth_tmp1}
    \end{align}
    Proceeding in this way, we can get
    \begin{equation}\label{eq:smooth_tmp2}
        \enVert{\frac{\partial\cR}{\partial W_1}-\frac{\partial\cR}{\partial V_1}}\le(L-1)R^{L-2}G\|W-V\|+R^{L-1}\|\nR(w)-\nR(v)\|.
    \end{equation}
    Since $\|z_i\|\le1$, $\ell'$ is $\beta$-Lipschitz, we have
    \begin{equation}\label{eq:smooth_tmp3}
        \|\nR(w)-\nR(v)\|\le\beta\|w-v\|\le\beta LR^{L-1}\|W-V\|,
    \end{equation}
    where the last inequality follows from a similar one-by-one replacement procedure as in \cref{eq:smooth_tmp1}. Combining \cref{eq:smooth_tmp2} and \cref{eq:smooth_tmp3}, we get for $R\ge1$,
    \begin{equation*}
        \enVert{\frac{\partial\cR}{\partial W_1}-\frac{\partial\cR}{\partial V_1}}\le\del{(L-1)R^{L-2}G+\beta LR^{2L-2}}\|W-V\|\le2LR^{2L-2}(\beta+G)\|W-V\|.
    \end{equation*}
    The same procedure can be done for other layers, and together
    \begin{equation*}
        \|\nR(W)-\nR(V)\|\le2L^2R^{2L-2}(\beta+G)\|W-V\|.
    \end{equation*}
\end{proof}

\begin{proof}[Proof of \Cref{lem:gd_escape}]
    Recall that if $W(t),W(t+1)\in B(R)$ and $\eta_t=1/\beta(R)$,
    \begin{align}
        \cR\del{W(t+1)}-\cR\del{W(t)} & \le \langle\nR\del{W(t)},-\eta_t\nR\del{W(t)}\rangle+\frac{\beta(R)\eta_t^2}{2}\enVert{\nR\del{W(t)}}^2 \nonumber \\
         & =-\frac{1}{2\beta(R)}\enVert{\nR\del{W(t)}}^2 \nonumber \\
         & =-\frac{\eta_t}{2}\enVert{\nR\del{W(t)}}^2. \label{eq:smooth_risk_drop}
    \end{align}

    Suppose $W(t)\in B(R)$ for all $t$. By \Cref{ass:init} and \cref{eq:smooth_risk_drop},
    \begin{equation*}
        \cR\del{W(1)}\le\cR\del{W(0)}-\frac{1}{2\beta(R)}\enVert{\nR\del{W(0)}}^2<\cR\del{W(0)}.
    \end{equation*}
    By \cref{eq:smooth_risk_drop}, gradient descent never increases the risk, and thus for all $t\ge1$, $\cR\del{W(t)}\le\cR\del{W(1)}<\cR\del{W(0)}$. In exactly the same way as in the proof of \Cref{lem:gf_unbounded}, one can show that there exists some constant $\epsilon(R)>0$, so that $\|\partial\cR/\partial W_1(t)\|_F\ge\epsilon(R)$ for all $t$. Invoking \cref{eq:smooth_risk_drop} again, we will get
    \begin{equation*}
        \cR\del{W(0)}\ge \sum_{t=0}^{\infty}\frac{1}{2\beta(R)}\epsilon(R)^2=\infty,
    \end{equation*}
    which is a contradiction. Therefore $W(t)$ must go out of $B(R)$ at some time.
\end{proof}

Next we prove \Cref{thm:gd_risk} and \ref{thm:gd_min_norm}. The proofs depend on several lemmas which are similar to the gradient flow ones. The following \Cref{lem:gd_unbounded} is similar to \Cref{lem:gf_unbounded}.

\begin{lemma}\label{lem:gd_unbounded}
    Under \Cref{ass:loss}, \ref{ass:init}, \ref{ass:smooth}, and \ref{ass:step_size}, gradient descent ensures that
    \begin{itemize}
        \item $\max_{1\le k\le L}\|W_k(t)\|_F$ is unbounded.
        \item $\sum_{t=0}^{\infty}\eta_t=\infty$.
        \item For any $R>0$, $\sum_{t:W(t)\in B(R)}^{}\eta_t<\infty$.
    \end{itemize}
\end{lemma}
\begin{proof}
    By \Cref{ass:step_size}, we always have that $W(t)\in B(R_t)$. Since $\beta(R_t)=2L^2R_t^{2L-2}(\beta+G)\ge R_t^{L-1}G$, we have for any $1\le k\le L$,
    \begin{align}
        \|W_k(t+1)\|_F & \le\|W_k(t)\|_F+\eta_t\enVert{\frac{\partial\cR}{\partial W_k(t)}}_F \nonumber \\
         & \le\|W_k(t)\|_F+\frac{1}{\beta(R_t)}\enVert{\frac{\partial\cR}{\partial W_k(t)}}_F \nonumber \\
         & \le\|W_k(t)\|_F+\frac{1}{\beta(R_t)}R_t^{L-1}G \nonumber \\
         & \le\|W_k(t)\|_F+1. \label{eq:gd_inc}
    \end{align}
    Moreover, \Cref{lem:gd_escape} shows that $R_t\to\infty$. Since $R_{t+1}=R_t$ as long as $W(t+1)\in B(R_t-1)$, $\max_{1\le k\le L}\|W_k(t)\|_F$ is unbounded.

    It then follows that for any $t$, by Cauchy-Schwarz,
    \begin{align*}
        \del{\sum_{\tau=0}^{t-1}\eta_{\tau}}\del{\sum_{\tau=0}^{t-1}\eta_{\tau}\enVert{\nR\del{W(\tau)}}^2} & \ge\del{\sum_{\tau=0}^{t-1}\eta_{\tau}\enVert{\nR\del{W(\tau)}}}^2\to\infty.
    \end{align*}
    since by \cref{eq:smooth_risk_drop},
    \begin{align*}
        \sum_{\tau=0}^{t-1}\eta_{\tau}\enVert{\nR\del{W(\tau)}}^2\le2\cR\del{W(0)}-2\cR\del{W(t)}\le2\cR\del{W(0)},
    \end{align*}
    we have $\sum_{t=0}^{\infty}\eta_t=\infty$.

    Since under \Cref{ass:smooth,ass:step_size} gradient descent never increases the risk, it can be shown in exactly the same as in the proof of \Cref{lem:gf_unbounded} that, for $W(t)\in B(R)$, $\|\partial\cR/\partial W_1(t)\|_F\ge\epsilon(R)$ for some constant $\epsilon(R)>0$. Invoking \cref{eq:smooth_risk_drop} again, we get that $\sum_{t:W(t)\in B(R)}^{}\eta_t<\infty$.
\end{proof}

The next lemma is an analogy to \Cref{lem:gf_align}.
\begin{lemma}\label{lem:gd_align}
    Under \Cref{ass:loss} and \ref{ass:smooth}, the gradient descent iterates satisfy the following properties:
    \begin{itemize}
        \item For any $1\le k\le L$, $\|W_k\|_F^2-\|W_k\|_2^2\le D+2\cR\del{W(0)}$.
        \item For any $1\le k<L$, $\langle v_{k+1},u_k\rangle^2\ge1-\nicefrac{D+3\cR\del{W(0)}+\|W_{k+1}(0)\|_2^2+\|W_k(0)\|_2^2}{\|W_{k+1}\|_2^2}$.
        \item Suppose $\max_{1\le k\le L}\|W_k\|_F\to\infty$, then $\envert{\left\langle \nicefrac{\wnn}{\prod_{k=1}^{L}\|W_k\|_F}, v_1\right\rangle}\to1$.
    \end{itemize}
\end{lemma}
\begin{proof}
    Recall that for any $W$,
    \begin{equation}\label{eq:gd_align_tmp1}
        W_{k+1}^{\top}\frac{\partial\cR}{\partial W_{k+1}}=W_{k+1}^{\top}\cdots W_L^{\top}\nR(\wnn)^{\top}W_1^{\top}\cdots W_k^{\top}=\frac{\partial\cR}{\partial W_k}W_k^{\top}.
    \end{equation}
    For gradient descent iterates, summing \cref{eq:gd_align_tmp1} from $0$ to $t-1$, we get
    \begin{align}
         & \quad W_{k+1}^{\top}(t)W_{k+1}(t)-W_{k+1}^{\top}(0)W_{k+1}(0)+\sum_{\tau=0}^{t-1}\eta_{\tau}^2\del{\frac{\partial\cR}{\partial W_{k+1}(\tau)}}^{\top}\del{\frac{\partial\cR}{\partial W_{k+1}(\tau)}} \nonumber \\
        = & \quad W_k(t)W_k^{\top}(t)-W_k(0)W_k^{\top}(0)+\sum_{\tau=0}^{t-1}\eta_{\tau}^2\del{\frac{\partial\cR}{\partial W_k(\tau)}}\del{\frac{\partial\cR}{\partial W_k(\tau)}}^{\top}. \label{eq:gd_spec}
    \end{align}
    For any $1\le k\le L$ and any $t$, let
    \begin{equation*}
        P_{k}(t)=\sum_{\tau=0}^{t-1}\eta_{\tau}^2\del{\frac{\partial\cR}{\partial W_k(\tau)}}\del{\frac{\partial\cR}{\partial W_k(\tau)}}^{\top},
    \end{equation*}
    and
    \begin{equation*}
        Q_{k}(t)=\sum_{\tau=0}^{t-1}\eta_{\tau}^2\del{\frac{\partial\cR}{\partial W_k(\tau)}}^{\top}\del{\frac{\partial\cR}{\partial W_k(\tau)}}.
    \end{equation*}
    We have $\|P_k(t)\|_2=\|Q_k(t)\|_2\le\tr\del{Q_k(t)}=\tr\del{P_k(t)}$. Moreover, invoking \cref{eq:smooth_risk_drop},
    \begin{align}
        \sum_{k=1}^{L}\tr\del{P_k(t)} & =\sum_{k=1}^{L}\sum_{\tau=0}^{t-1}\eta_{\tau}^2\enVert{\frac{\partial\cR}{\partial W_k(\tau)}}_F^2 \nonumber \\
         & =\sum_{\tau=0}^{t-1}\eta_{\tau}^2\enVert{\nR\del{W(\tau)}}^2 \nonumber \\
         & \le\sum_{\tau=0}^{t-1}\eta_{\tau}\enVert{\nR\del{W(\tau)}}^2 \nonumber \\
         & \le2\cR\del{W(0)}-2\cR\del{W(t)} \nonumber \\
         & \le2\cR\del{W(0)}. \label{eq:gd_grad2}
    \end{align}

    Still let $\sigma_k(t)$, $u_k(t)$ and $v_k(t)$ denote the first singular value, left singular vector and right singular vector of $W_k(t)$. We can then proceed basically in the same way as in the proof of \Cref{lem:gf_align}. For example, \cref{eq:svd_1} becomes
    \begin{align}\label{eq:gd_svd_1}
        \sigma_k^2(t)\ge\sigma_{k+1}^2(t)-\|A_{k,k+1}(t)\|_2-\|P_{k}(t)\|_2\ge\sigma_{k+1}^2(t)-\|A_{k,k+1}(t)\|_2-\tr\del{P_k(t)},
    \end{align}
    while \cref{eq:ind_2} becomes
    \begin{align}\label{eq:gd_ind_1}
        \|W_k(t)\|_F^2=\|W_{k+1}(t)\|_F^2+\|W_k(0)\|_F^2-\|W_{k+1}(0)\|_F^2-\tr\del{P_k(t)}+\tr\del{Q_{k+1}(t)}.
    \end{align}
    Summing \cref{eq:gd_svd_1} and \cref{eq:gd_ind_1} from $k$ to $L-1$, and invoke \cref{eq:gd_grad2}, we get
    \begin{align*}
        \|W_k(t)\|_F^2-\|W_k(t)\|_2^2 &\le D-\tr\del{P_k(t)}+\tr\del{Q_L(t)}+\sum_{k'=k}^{L-1}\tr\del{P_{k'}(t)}\le D+2\cR\del{W(0)}.
    \end{align*}

    To prove singular vectors get aligned, we can still proceed in nearly the same way as in the proof of \Cref{lem:gf_align}. \cref{eq:svd_2} becomes
    \begin{align}\label{eq:gd_svd_2}
        u_k^{\top}W_{k+1}^{\top}W_{k+1}u_k\ge\sigma_k^2-\|W_k(0)\|_2^2-\|Q_{k+1}(t)\|_2,
    \end{align}
    while \cref{eq:svd_3} becomes
    \begin{align}\label{eq:gd_svd_3}
        u_k^{\top}W_{k+1}^{\top}W_{k+1}u_k\le \langle u_k,v_{k+1}\rangle^2\sigma_{k+1}^2+D+2\cR\del{W(0)}.
    \end{align}
    Combining \cref{eq:gd_svd_2} and \cref{eq:gd_svd_3}
    \begin{align}\label{eq:gd_dp_prelim}
        \sigma_k^2\le \langle u_k,v_{k+1}\rangle^2\sigma_{k+1}^2+D+2\cR\del{W(0)}+\|Q_{k+1}(t)\|_2+\|W_k(0)\|_2^2.
    \end{align}
    Similar to \cref{eq:gd_svd_2}, we can get
    \begin{align*}
        \sigma_k^2\ge v_{k+1}^{\top}W_kW_k^{\top}v_{k+1}\ge\sigma_{k+1}^2-\|W_{k+1}(0)\|_2^2-\|P_k(t)\|_2,
    \end{align*}
    and thus \cref{eq:sv_ratio} becomes
    \begin{align}\label{eq:gd_sv_ratio}
        \frac{\sigma_k^2}{\sigma_{k+1}^2}\ge1-\frac{\|W_{k+1}(0)\|_2^2+\|P_k(t)\|_2}{\sigma_{k+1}^2}.
    \end{align}
    Combining \cref{eq:gd_dp_prelim} and \cref{eq:gd_sv_ratio}, we get
    \begin{align*}
        \langle u_k,v_{k+1}\rangle^2\ge1-\frac{D+\|W_k(0)\|_2^2+\|W_{k+1}(0)\|_2^2+3\cR\del{W(0)}}{\sigma_{k+1}^2}.
    \end{align*}

    The final claim of \Cref{lem:gd_align} can be proved in exactly the same way as \Cref{lem:gf_align}.
\end{proof}

\begin{proof}[Proof of \Cref{thm:gd_risk}]
    Summing \cref{eq:gd_ind_1}, we know that for any two different layers $j>k$,
    \begin{align*}
        \|W_k(t)\|_F^2-\|W_j(t)\|_F^2=\|W_k(0)\|_F^2-\|W_j(0)\|_F^2-\tr\del{P_k(t)}+\tr\del{Q_j(t)}.
    \end{align*}
    Recall \cref{eq:gd_grad2}, we know that
    \begin{align}\label{eq:diff_norm2}
        \envert{\del{\|W_k(t)\|_F^2-\|W_j(t)\|_F^2}-\del{\|W_k(0)\|_F^2-\|W_j(0)\|_F^2}}\le2\cR\del{W(0)}.
    \end{align}
    In other words, the difference between the squares of Frobenius norms of any two layers is still bounded.

    The proof then goes in the same way as the proof of \Cref{thm:gf_risk}. Suppose the risk is always above $\epsilon>0$. Then there exists some $c(\epsilon)>0$ such that $\|\nR(\wnn)\|\ge c(\epsilon)$. By \Cref{lem:gd_align}, there exists some $C$ such that if $\min_{1\le k\le L}\|W_k(t)\|_F>C$, $\|W_L(t)\cdots W_2(t)\|\ge C^L/2$. By \cref{eq:diff_norm2} and \Cref{lem:gd_unbounded}, $\sum_{t:\|W_k(t)\|_F\le C\textrm{ for some }k}\eta_t$ is finite. On the other hand, by \Cref{lem:gd_unbounded}, $\sum_{i=0}^{\infty}\eta_t=\infty$, and thus $\sum_{t:\|W_k(t)\|_F>C\textrm{ for all }k}^{}\eta_t=\infty$. Therefore we have, by invoking \cref{eq:smooth_risk_drop},
    \begin{align*}
        2\cR\del{W(0)} & \ge \sum_{t=0}^{\infty}\eta_t\enVert{\cR\del{W(t)}}^2 \\
         & \ge \sum_{t=0}^{\infty}\eta_t\enVert{\frac{\partial\cR}{\partial W_1(t)}}^2 \\
         & \ge c(\epsilon)\frac{C^L}{2}\sum_{t:\|W_k(t)\|_F>C\textrm{ for all }k}^{}\eta_t \\
         & =\infty,
    \end{align*}
    which is a contradiction. Therefore $\cR\del{W(t)}\to0$, and since it has no finite optimum, $\|W_k\|_F\to\infty$. The other results follow from \Cref{lem:gd_unbounded}.
\end{proof}

\begin{proof}[Proof of \Cref{thm:gd_min_norm}]
    Recall that
    \begin{equation*}
        \frac{\partial\cR}{\partial W_1}=W_2^{\top}\cdots W_L^{\top}\nR(\wnn),
    \end{equation*}
    and thus
    \begin{align*}
        \|W_1(t+1)\|_F^2 & =\|W_1(t)\|_F^2-2\eta_t \left\langle W_1(t),\frac{\partial\cR}{\partial W_1(t)}\right\rangle+\eta_t^2\enVert{\frac{\partial\cR}{\partial W_1(t)}}_F^2 \\
         & =\|W_1(t)\|_F^2-2\eta_t \left\langle\wnn(t),\nR\del{\wnn(t)}\right\rangle+\eta_t^2\enVert{\frac{\partial\cR}{\partial W_1(t)}}_F^2.
    \end{align*}
    If $\langle\wnn,z_i\rangle\ge0$ for all $i$, then $\|W_1(t+1)\|_F\ge\|W_1(t)\|_F$.

    Also recall that $\Pi_{\perp}W_1(t)$ denote the projection of rows of $W_1(t)$ onto $\baru^{\perp}$, the orthogonal complement of $\mathrm{span}(\baru)$. We have
    \begin{align}
        \|\Pi_{\perp}W_1(t+1)\|_F^2 & \le\|\Pi_{\perp}W_1(t)\|_F^2-2\eta_t \left\langle\Pi_{\perp}W_1(t),\frac{\partial\cR}{\partial W_1(t)}\right\rangle+\eta_t^2\enVert{\frac{\partial\cR}{\partial W_1(t)}}_F^2 \nonumber \\
         & =\|\Pi_{\perp}W_1(t)\|_F^2-2\eta_t \left\langle\Pi_{\perp}\wnn(t),\nR\del{\wnn(t)}\right\rangle+\eta_t^2\enVert{\frac{\partial\cR}{\partial W_1(t)}}_F^2. \label{eq:gd_mn_tmp1}
    \end{align}
    Invoking \cref{eq:smooth_risk_drop} again gives
    \begin{align}\label{eq:gd_mn_tmp2}
        \eta_t^2\enVert{\frac{\partial\cR}{\partial W_1(t)}}_F^2\le\eta_t\enVert{\nR\del{W(t)}}^2\le2\del{\cR\del{W(t)}-\cR\del{W(t+1)}}.
    \end{align}

    The proof then goes in almost the same way as the proof of \Cref{thm:gf_min_norm}. For any $\epsilon>0$, we can find some large enough time $t_0$, such that for any $t\ge t_0$,
    \begin{enumerate}
        \item $\|\Pi_{\perp}W_1(t)\|_F/\|W_1(t)\|F\ge\epsilon$ implies that $\left\langle\Pi_{\perp}\wnn(t),\nR\del{\wnn(t)}\right\rangle\ge0$.
        \item $\langle \wnn(t),z_i\rangle\ge0$ for all $i$, and thus $\|W_1(t+1)\|_F\ge\|W_1(t)\|_F$.
        \item $\|W_1(t)\|_F\ge\nicefrac{1+\sqrt{2\cR(W(0))}}{\epsilon}$.
    \end{enumerate}

    Suppose at some time $t_1\ge t_0$, $\|\Pi_{\perp}W_1(t_1)\|_F/\|W_1(t_1)\|_F\ge\epsilon$. As long as this still holds, in light of bullet (1) above, \cref{eq:gd_mn_tmp1} and \cref{eq:gd_mn_tmp2}, $\|\Pi_{\perp}W_1\|_F^2$ will increase by at most $2\cR\del{W(t_1)}\le2\cR\del{W(0)}$. On the other hand, $\|W_1\|_F\to\infty$, and thus there exists some $t_2>t_1$ such that $\|\Pi_{\perp}W_1(t_2)\|_F/\|W_1(t_2)\|F<\epsilon$.

    Let $t_3$ denote the smallest time after $t_2$ such that $\|\Pi_{\perp}W_1(t_3)\|_F/\|W_1(t_3)\|_F\ge\epsilon$ (if it exists). Recall that $\|W_1(t+1)\|_F\le\|W_1(t)\|_F+1$ for any $t\ge0$, and $\|W_1(t+1)\|_F\ge\|W_1(t)\|_F$ for any $t\ge t_0$, we have
    \begin{equation*}
        \frac{\|\Pi_{\perp}W_1(t_3)\|_F}{\|W_1(t_3)\|_F}\le \frac{\|\Pi_{\perp}W_1(t_3)\|_F}{\|W_1(t_3-1)\|_F}\le \frac{\|\Pi_{\perp}W_1(t_3-1)\|_F+1}{\|W_1(t_3-1)\|_F}<\epsilon+\frac{1}{\|W_1(t_3-1)\|_F}.
    \end{equation*}
    After $t_3$, $\|\Pi_{\perp}W_1\|_F^2$ will increase by at most $2\cR\del{W(0)}$, and thus $\|\Pi_{\perp}W_1\|_F$ will increase by at most $\sqrt{2\cR\del{W(0)}}$. Therefore, for any $t_4\ge t_3$, as long as $\|\Pi_{\perp}W_1(t_4)\|_F/\|W_1(t_4)\|_F\ge\epsilon$, we have
    \begin{align*}
        \frac{\|\Pi_{\perp}W_1(t_4)\|_F}{\|W_1(t_4)\|_F} & \le \frac{\|\Pi_{\perp}W_1(t_4)\|_F}{\|W_1(t_3)\|_F} \\
         & \le \frac{\|\Pi_{\perp}W_1(t_3)\|_F+\sqrt{2\cR\del{W(0)}}}{\|W_1(t_3)\|_F} \\
         & \le\epsilon+\frac{1}{\|W_1(t_3-1)\|_F}+\frac{\sqrt{2\cR\del{W(0)}}}{\|W_1(t_3)\|_F}\le2\epsilon,
    \end{align*}
    since $\|W_1(t)\|_F\ge\nicefrac{1+\sqrt{2\cR(W(0))}}{\epsilon}$ after $t_0$. In other words,
    \begin{equation*}
        \lim\sup_{t\to\infty}\frac{\|\Pi_{\perp}W_1\|_F}{\|W_1\|_F}\le2\epsilon.
    \end{equation*}
    Since $\epsilon$ is arbitrary, we have
    \begin{equation*}
        \lim\sup_{t\to\infty}\frac{\|\Pi_{\perp}W_1\|_F}{\|W_1\|_F}=0,
    \end{equation*}
    and thus $\lim_{t\to\infty}\envert{\langle v_1,\baru\rangle}=1$.

\end{proof}

\begin{proof}[Proof of \Cref{fact:gd_min_norm_2}]
  The proof is analogous to that of \Cref{fact:gf_min_norm_2},
  except using \Cref{thm:gd_min_norm} in place of \Cref{thm:gf_min_norm}.
\end{proof}

\end{document}